\documentclass{article}


\usepackage[main,final]{neurips_2026}

\usepackage{xcolor}
\newif\ifdraft
\draftfalse       

\ifdraft
    \newcommand{\jiangnan}[1]{\textcolor{chestnut}{\textbf{[#1]}}}
    \newcommand{\heng}[1]{{{\textcolor{chestnut}{[Heng: #1]}}}}
    \newcommand{\hq}[1]{\textcolor{orange}{\textbf{[#1]}}}
    \newcommand{\hanqi}[1]{\textcolor{blue}{#1}}
    \newcommand{\zy}[1]{{{\textcolor{purple}{[Zhenyi: #1]}}}}
    \newcommand{\yh}[1]{{{\textcolor{red}{[Yulan: #1]}}}}
    \newcommand{\neu}[1]{\textcolor{blue}{#1}}

\else
    \newcommand{\dc}[1]{}
    \newcommand{\jiangnan}[1]{}
    \newcommand{\heng}[1]{}
    \newcommand{\hq}[1]{}
    \newcommand{\zy}[1]{}
    \newcommand{\yh}[1]{}
    \newcommand{\hanqi}[1]{}
    \newcommand{\neu}[1]{}{}

\fi

\newcommand{\method}{ComprExIT}

\usepackage[utf8]{inputenc}
\usepackage[T1]{fontenc}

\usepackage{microtype}
\usepackage{graphicx}
\usepackage{wrapfig}
\usepackage{needspace}
\usepackage{subcaption}
\usepackage{placeins}
\usepackage{tabularx}
\usepackage{caption}
\usepackage{booktabs}
\usepackage{multirow}
\usepackage{colortbl}
\usepackage[table]{xcolor}
\usepackage{makecell}
\usepackage{nicefrac}

\usepackage{hyperref}
\usepackage{url}

\usepackage{amsmath}
\usepackage{amssymb}
\usepackage{amsfonts}
\usepackage{bm}
\usepackage{mathtools}
\usepackage{amsthm}
\usepackage{enumitem}
\usepackage[most]{tcolorbox}

\usepackage[capitalize,noabbrev]{cleveref}

\theoremstyle{plain}
\newtheorem{theorem}{Theorem}[section]
\newtheorem{proposition}[theorem]{Proposition}

\theoremstyle{definition}

\theoremstyle{remark}

\newtcolorbox{keyquestionbox}{
  enhanced,
  breakable,
  colback=black!3,
  colframe=black!70,
  boxrule=0.6pt,
  arc=2pt,
  left=6pt,
  right=6pt,
  top=4pt,
  bottom=4pt,
  width=\linewidth,
  before skip=0.45em,
  after skip=0.45em,
}

\usepackage[textsize=tiny]{todonotes}
\definecolor{chestnut}{cmyk}{0, 0.7808, 0.4429, 0.1412}

\title{Fix the Structural Bottleneck: Context Compression via Explicit Information Transmission}

\author{%
  Jiangnan Ye \\
  King's College London, UK \\
  \And
  Hanqi Yan\thanks{Corresponding authors: \texttt{hanqi.yan@kcl.ac.uk}, \texttt{yulan.he@kcl.ac.uk}.} \\
  King's College London, UK \\
  \And
  Zhenyi Shen \\
  King's College London, UK \\
  \AND
  Heng Chang \\
  Unity China, China \\
  \And
  Ye Mao \\
  Imperial College London, UK \\
  \And
  Yulan He$^{*}$ \\
  King's College London, UK \\
  The Alan Turing Institute, UK \\
}

\begin{document}
\hypersetup{hidelinks}
\maketitle

\vspace{-2em}
\begin{abstract}
\vspace{-3pt}
Long-context LLM agents often struggle with growing token, memory, and latency costs, making efficient context compression essential for practical deployment. Existing LLM-as-a-compressor methods remain noticeably inferior to using the full context. We find that this gap partly stems from their inability to preserve contextual information effectively. In this work, we revisit context compression from a structural perspective and identify two key bottlenecks in standard LLM-based compressors: limited coordination among compression tokens during information aggregation, and layerwise dilution that weakens useful signals from intermediate hidden states. To address these limitations, we propose \textbf{\method{}} (Context \textbf{Compr}ession
via \textbf{Ex}plicit \textbf{I}nformation \textbf{T}ransmission), a new context compression framework based on explicit information transmission. \method{} adaptively aggregates features into anchors across frozen LLM layers, then allocates information from anchors to compression slots through a globally coordinated transport plan. Experiments on 12 datasets show that ComprExIT consistently outperforms strong soft-compression baselines, improving average F1 by up to 18.5\%, while adding only $\sim$1\% trainable parameters and achieving over 2$\times$ faster compression than the fastest baselines. Code is available at \url{https://github.com/Jiangnan0522/ComprExIT}.

\end{abstract}

\vspace{-1em}
\section{Introduction} 
\vspace{-0.5em}
\label{section:intro}
\begin{wrapfigure}{r}{0.45\textwidth}
    \vspace{-1.4em}
    \centering
    \includegraphics[width=\linewidth]{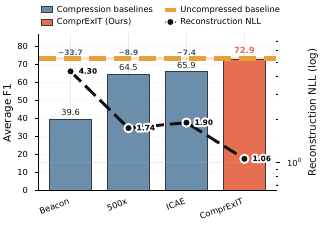}
    \caption{Performance (Average F1 over six datasets, Llama-3.2-3B) of existing context compression methods and their ability to preserve contextual information (reconstruction negative log-likelihood).}
    \label{fig:intro}
    \vspace{-1em}
\end{wrapfigure}

Large Language Models (LLMs) have become unprecedentedly powerful, and these capabilities come with ever-growing context lengths~\citep{team2026kimi, team2026qwen3, deepseekai2026deepseekv4}. Long contexts rapidly consume tokens, saturate the effective context window, and enlarge the KV cache, incurring substantial latency and memory overhead~\citep{liu2025comprehensive, xiao2025duoattention}. These costs can ultimately limit the practical performance of LLMs on downstream tasks~\citep{li2025longcontext, liu-etal-2024-lost}. 

One emerging technique for addressing these problems is \textit{Context Compression} (CC)~\citep{li-etal-2025-prompt}, which compresses the original context into a fraction of its length. Current CC methods largely follow the LLM-as-a-compressor paradigm: they repurpose an LLM as a compressor by modifying its internal computation through continual training~\citep{mu2023learning, ge2024incontext, li-etal-2025-500xcompressor, zhang2025long, deng2025unigist}. Specifically, as shown in \cref{fig:framework}-left, context tokens and a small set of gist tokens are fed into the LLM, which is trained to encode contextual information into these gist tokens.
Although many design choices have been explored within this paradigm, existing methods still fall well short of the uncompressed baseline (\cref{fig:intro} bars). We hypothesize that a key reason is their limited ability to retain contextual information during compression. To investigate this, we use the gist tokens produced by different compressors as inputs and measure the decoder LLM's reconstruction negative log-likelihood (NLL) on the original context\footnote{Experimental details are in Appendix~\ref{appendix:exp-recons}.}, which serves as a proxy for how much information the compressed representation preserves. As shown in \cref{fig:intro} (black line), methods with worse downstream performance (lower F1) also yield higher reconstruction NLL, suggesting that insufficient information retention is a critical bottleneck of existing approaches.




Since the current LLM-as-a-compressor paradigm is built on the standard LLM architecture, this gap in information retention calls for a structural re-examination. By tracing how information flows in the compressor LLM, we identify two limitations. In the \textit{width dimension}, LLMs rely on self-attention to condense information from all context tokens into the gist tokens. Yet this process suffers from a \textbf{Lack of Allocation}: each gist token gathers information independently, without coordinating with the others about which parts of the context they represent. Consequently, multiple gist tokens may focus on the same regions while others remain under-covered, leading to information loss. In the \textit{depth dimension}, LLMs propagate information in a layer-wise accumulative manner. As forward propagation proceeds, fine-grained evidence captured in earlier layers can become increasingly entangled with later, more abstract, generation-oriented representations~\citep{zhang2024comprehensive, skean2025layer}. As a result, early-layer features become harder to recover from the final gist-token states, leading to \textbf{Information Dilution}~\citep{chen2026attnres}.


The analysis above motivates a different formulation: instead of tuning an LLM into a compressor, we treat it as a \textit{frozen} feature extractor and directly use its hidden states for compression. We propose \method{} (Context \textbf{Compr}ession
via \textbf{Ex}plicit \textbf{I}nformation \textbf{T}ransmission), a new context compression framework that decomposes compression into two stages, each addressing one of the limitations above. The \textbf{Widthwise Stage} is built around an optimal-transport-based~\citep{Villani2008OptimalTO} transmission plan that globally allocates information across tokens and explicitly coordinates how contextual information is distributed to compression slots\footnote{Compression slots are the output representations that hold the compressed context. They play the role of gist tokens in LLM-as-a-compressor methods.}. The \textbf{Depthwise Stage} establishes weighted shortcuts across layers, where fine-grained features extracted at different levels are directly aggregated, thereby mitigating information dilution across layers.

Our experiments across 12 datasets show that \method{} significantly outperforms state-of-the-art CC methods, with an average F1 improvement of 18.5\%. It is also lightweight and efficient, adding only $\sim$1\% parameters while compressing over $2\times$ faster than the fastest baselines. Further analysis shows that our key design components improve information preservation and thereby boost compression performance.

We summarize our contributions as follows:
\begin{itemize}[topsep=0pt]
    \setlength{\itemsep}{0.2em}
    \setlength{\parskip}{0pt}
    \item We analyze why existing LLM-as-a-compressor methods lose contextual information, and identify two key bottlenecks: \textit{lack of allocation} in the width dimension and \textit{information dilution} in the depth dimension.
    \item We propose \method{}, a lightweight context compression framework that treats a frozen LLM as a feature extractor and addresses these bottlenecks with a coordinated widthwise transmission plan and depthwise layer aggregation.
    \item We show that \method{} consistently outperforms prior context compression methods on 12 datasets, while remaining parameter-efficient and faster at compression.
\end{itemize}
\section{Related Work}
\vspace{-0.5em}
\label{section:related-work}
Existing context compression methods can be divided into two categories: token pruning and soft context compression, which operate in discrete and continuous spaces respectively.

\textbf{Token pruning.} Token pruning methods reduce context length by pruning input tokens according to estimated importance. Representative approaches include SelectiveContext~\citep{li-etal-2023-compressing} and LLMLingua-style methods~\citep{jiang-etal-2023-llmlingua, jiang-etal-2024-longllmlingua}, which score tokens/spans using a language model (e.g., via self-information or related salience measures) and retain only the most informative parts of the context. LLMLingua-2~\citep{pan-etal-2024-llmlingua} instead distills a lightweight classifier from a stronger LLM (e.g., GPT-4) to decide which tokens to keep. EFPC~\citep{cao2025efpc} further unifies task-aware and task-agnostic compression within a single framework. Despite their efficiency gains, hard methods remain discrete and lossy, often facing bounded compression ratios and limited expressiveness compared with soft compression in continuous space.

\textbf{Soft context compression.} Soft context compression offers greater flexibility and expressiveness than discrete token removal at high compression ratios, and our method belongs to this category. \citeauthor{wingate-etal-2022-prompt} presented an early attempt to compress context tokens into a single vector, focusing on representation learning. AutoCompressor~\citep{chevalier-etal-2023-adapting} builds LLMs into compressors inspired by the Recurrent Memory Transformer~\citep{bulatov2022recurrent}, accumulatively compressing contexts into summary tokens.~\citeauthor{mu2023learning} introduce gist tokens as an information bottleneck by modifying attention masks, 
establishing the LLM-as-a-compressor paradigm. ICAE~\citep{ge2024incontext} further simplifies this paradigm into an encoder–decoder framework, which has since become a widely adopted paradigm.
Building upon ICAE, 
500×~\citep{li-etal-2025-500xcompressor}, Activation Beacon~\citep{zhang2025long}, and UniGist~\citep{deng2025unigist} pass cached key–value states to the decoder, together with carefully designed gist-token attention masks, to obtain more informative compressed representations. 
\citeauthor{deng-etal-2025-silver} systematically study the effects of different design components within the encoder–decoder compression framework. EPL~\citep{zhao-etal-2025-position} further improves compression performance by adjusting the positional encodings of gist tokens. xRAG~\citep{cheng2024xrag} explores converting representations from text embedding models into compression tokens. SAC~\citep{liu2026autoencodingfree} examines the necessity of the auto-encoding training objective, demonstrating that effective compression can be achieved using only a text completion objective. \citet{tang2026comi} dynamically assign compression budget through coarse-to-fine compression.
Different from the existing methods that use the internal self-attention mechanisms of LLMs to perform compression, our method adopts a fundamentally different paradigm by decoupling compression from the LLM architecture and formulating it as an explicit information transmission problem over frozen hidden states.

\begin{figure*}[!t]
    \centering
    \begin{minipage}[t]{\textwidth}
        \centering
        \includegraphics[width=\linewidth]{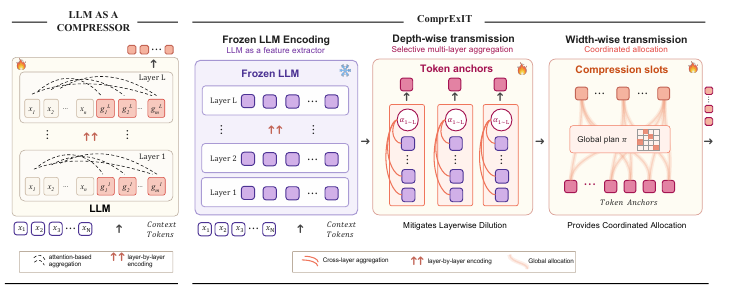}
    \end{minipage}
    \caption{A comparison between existing LLM-as-a-compressor methods (left) and \method{} (right). \textbf{Existing methods} introduce gist tokens that are iteratively encoded across layers by the self-attention in the LLMs. \textbf{\method{}} instead leverages the layerwise hidden states of the context tokens encoded in a forward pass. The hidden states are selectively aggregated into token anchors, which are then transmitted to the compression tokens through a coordinated transmission plan.}
    \label{fig:framework}
\end{figure*}
\vspace{-0.5em}

\vspace{-0.5em}
\section{Preliminary}
\vspace{-0.5em}
\label{sec:info_preservation}

\textbf{Context compression formulation.} 
Context compression consists of a compressor and a decoder. Given a context sequence $\bm{x}=(x_1,\ldots,x_N)\in\mathbb{I}^{N}$, the compressor produces a compact representation $Z=f_\phi(\bm{x})\in\mathbb{R}^{K\times d}$ with $K\ll N$. We refer to the $K$ vectors of $Z$ as \emph{compression slots}. The decoder then conditions on $Z$ to generate outputs.


\textbf{LLM-as-a-compressor.}
Most existing methods instantiate $f_\phi$ by turning an LLM into a compressor. Given a context sequence $\bm{x}=(x_1,\ldots,x_N)$ and $K$ learnable compression embeddings $G=(\bm{g}_1,\ldots,\bm{g}_K)\in\mathbb{R}^{K\times d}$, often called gist tokens, the model feeds the concatenated hidden-state sequence $[\mathrm{Emb}(\bm{x});G]$ into a Transformer:
\begin{equation}
H^{(0)}=[\mathrm{Emb}(\bm{x});G]\in\mathbb{R}^{(N+K)\times d}, \qquad
H^{(\ell)}=\mathrm{Block}_\ell\bigl(H^{(\ell-1)}\bigr), \quad \ell=1,\ldots,L.
\end{equation}
Here, $H^{(\ell)}$ contains hidden states for all positions in the concatenated sequence. The compressed representation is defined by the compression-token positions in the last-layer output:
\begin{equation}
H^{(L)}=\bigl[H^{(L)}_{\mathrm{ctx}}; H^{(L)}_G\bigr],  \qquad Z = H^{(L)}_G.
\end{equation}
Thus, compression is constrained by the same mechanisms used in LLM computation: self-attention for widthwise aggregation and residual layer propagation for depthwise information transfer.
\vspace{-0.5em}
\section{Limitations and Observations}
\vspace{-0.5em}

\textbf{Widthwise limitation: lack of allocation.}
At the $\ell$-th layer, the attention of gist tokens over the context tokens and the gist tokens can be written as
\begin{equation}
A^{(\ell)} =
\mathrm{softmax}
\left(
\frac{Q^\ell_G\,[K^\ell_X;K^\ell_G]^\top}{\sqrt{d_h}} + M
\right)
\in \mathbb{R}^{K\times (N+K)},
\label{eq:gist_attn}
\end{equation}
where $M$ is the attention mask. Here, each row $A^{(\ell)}_{i,:}$ is normalized independently, meaning that each gist token decides where to attend on its own. As a result, standard self-attention provides no mechanism to allocate how gist tokens absorb information across the context: several tokens may concentrate on the same region, while other regions remain weakly covered. To test whether this happens in practice, we analyze the attention distributions of gist tokens during compression\footnote{Experiment specifications for examining the allocation patterns can be found in Appendix~\ref{appendix:exp-alloc}.}, which reveal where each token gathers information from.
\begin{wrapfigure}{r}{0.58\columnwidth}
    \centering
    \vspace{-0.8em}
    \begin{subfigure}[t]{0.48\linewidth}
        \centering
        \includegraphics[width=\linewidth]{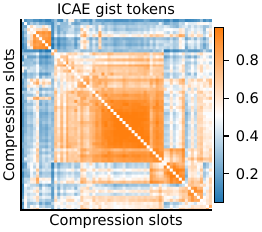}
    \end{subfigure}\hfill
    \begin{subfigure}[t]{0.48\linewidth}
        \centering
        \includegraphics[width=\linewidth]{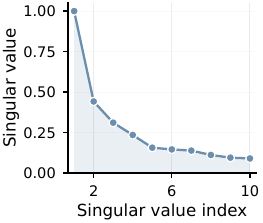}
    \end{subfigure}
    \caption{Two views of gist-token aggregation patterns in ICAE: Pearson correlation between gist-token attention distributions (the diagonal, which is always 1, is masked in white for readability) and the singular-value spectrum of the attention aggregation matrix.}
    \label{fig:observations_allocation}
    \vspace{-1.3em}
\end{wrapfigure}
As shown in \cref{fig:observations_allocation}, the attention distributions of different gist tokens are highly correlated, indicating that they often draw information from similar parts of the context. The full attention map from gist tokens to context tokens is shown in Appendix \cref{fig:width-icae}, where the bright and faint vertical bands mark context tokens that are attended to, or ignored, by many gist tokens together. Moreover, the aggregation matrix exhibits a steep singular-value decay, which points to low diversity in the collection behaviour. Together, these findings suggest that many gist tokens capture redundant rather than complementary content. We therefore attribute this pattern to a lack of allocation in self-attention: without explicit coordination, gist tokens cannot spread out to cover distinct context regions and instead collapse onto the same salient areas.

\vspace{-0.5em}
\textbf{Depthwise limitation: layerwise dilution.} 
Let $G^{(\ell)}=H^{(\ell)}_G$ denote the compression-token states at layer $\ell$. In most LLM-as-a-compressor methods, the decoder only receives the final representation $Z=G^{(L)}$. This design implicitly assumes that task-relevant information gathered at earlier layers is preserved through all subsequent updates. In practice, however, features that are salient at an intermediate layer can be attenuated as the network continues forwarding. To characterize how much usable information remains at depth $\ell$, given the task query $q$ and the target output $y$, we define
\begin{equation}
S_\ell =
\max_{\psi\in\Psi}
\mathrm{Score}
\left(
p_\theta(y\mid q,\psi(G^{(\ell)}))
\right),
\label{eq:layer_score}
\end{equation}
where $\psi$ is a lightweight readout or compression module, and $\mathrm{Score}(\cdot)$ denotes the downstream task metric. If $S_\ell$ reaches its maximum at some intermediate layer and then drops as depth approaches $L$, this indicates a layerwise dilution bottleneck: information that is accessible earlier is not fully retained in the final state. Motivated by this hypothesis, we examine how information evolves along both the width and depth dimensions when an LLM is tuned as a compressor.
We train a probing compressor that uses the hidden states from a single chosen LLM layer. The compressor consists of $4\times$ mean pooling followed by a 2-layer MLP. We train the models on SQuAD and HotpotQA and report the decoder F1. Each point in \Cref{fig:observations_layers} is thus an empirical estimate of $S_\ell$ in \cref{eq:layer_score}, where we train the same probe $\psi$ separately for each layer and use F1 as $\mathrm{Score}(\cdot)$.

\begin{wrapfigure}{r}{0.45\columnwidth}
    \centering
    \vspace{-1.0em}
    \includegraphics[width=\linewidth]{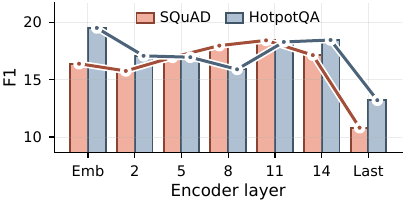}
\caption{Performance of using each LLM layer for compression.}
\label{fig:observations_layers}
\vspace{-0.6em}
\end{wrapfigure}
As shown in \Cref{fig:observations_layers}, hidden states from many layers provide useful compression features, but performance is not monotonic across layers, and the final layer performs the worst. This supports the view that useful features formed in earlier layers are gradually diluted rather than faithfully preserved through subsequent transformations. Moreover, contributions of layers are task-dependent: early and late layers contribute more to reasoning-intensive QA, whereas middle layers are more effective for extractive QA. These suggest that a strong compression method should both preserve and adaptively exploit features from different layers according to the context and task.

\section{Method}
\label{sec:method}

We present our method in computational order: from depthwise to widthwise computation.

\subsection{Depthwise Transmission to Token Anchors}

We first address the depthwise limitation discussed above, namely the layerwise dilution bottleneck. In LLM-as-a-compressor methods, useful information must remain accessible after repeated layer updates before it reaches the final compression states, even though our preliminary analysis suggests that the accessible utility can peak at intermediate layers. Therefore, we directly read from hidden states across depths and construct a token anchor at each token position by adaptively selecting and aggregating multi-layer features.
Given the hidden states $\{\bm{h}_t^{(\ell)}\}_{\ell=1}^L$ of token $t$, we first compute a lightweight structural context by mixing layer representations:
\begin{equation}
\bm{c}_t = \sum_{\ell=1}^L \pi_\ell \bm{h}_t^{(\ell)}, 
\qquad \sum_{\ell=1}^L \pi_\ell = 1,
\end{equation}
where $\pi_\ell$ denotes learned layer-level structure priors, inspired by Denseformer~\citep{pagliardini2024denseformer}. Conditioned on this context, we perform a token-wise gating attention over layers:
\begin{equation}
\bm{\tilde{h}}_t
=
\sum_{\ell=1}^L
\operatorname{softmax}_{\ell}
\left(
\frac{
\langle W_q \bm{c}_t,\; W_k \bm{h}_t^{(\ell)} + \bm{e}_\ell \rangle
}{\tau}
\right)
W_v \bm{h}_t^{(\ell)} ,
\end{equation}
where $W_q$, $W_k$, and $W_v$ are learnable query, key, and value projection matrices, $\bm{e}_\ell$ is a learnable layer embedding that encodes the layer index, and $\tau$ is the temperature. The resulting $\bm{\tilde{h}}_t$ serves as the token anchor. It can be viewed as a gated multi-layer readout from the frozen LLM to the compression interface, rather than an accumulated state that must survive until the last layer. In this way, features that are useful at intermediate depths remain directly accessible to the compressor, which mitigates the layerwise dilution bottleneck. Moreover, because the gating is conditioned on each token, \method{} can preserve different granularities of information for different tokens and tasks.

\subsection{Widthwise Transmission to Compression Slots}
Given the token anchors aggregated across layers, we next address the widthwise limitation discussed above, namely the lack of allocation among compression tokens. Instead of letting each slot independently retrieve information as in self-attention, we explicitly allocate information from $N$ token anchors to $K$ compression slots through a shared transmission plan. Formally, we seek $\Pi \in \mathbb{R}_{+}^{N \times K}$, where $\Pi_{t,k}$ denotes how much information token anchor $t$ sends to compression slot $k$. 

\textbf{Utility matrix.} To coordinate all possible sender-slot paths globally, we first construct a \emph{utility matrix} $U_{t,k}$ that measures how beneficial it is to transmit information from anchor $t$ to slot $k$. We treat each token anchor as a \emph{sender}. For each compression slot, we build a corresponding \emph{receiver} and assign it a local field of senders so that the initial slot semantics remain aligned with the original token order. Concretely, we uniformly partition the token-anchor sequence into local fields $\mathcal{F}_k$, average the anchors within each field to obtain the receiver representation, and then compute the utility matrix using the cosine similarity between projected sender and receiver representations:
\begin{equation}
\bm{r}_{k} = \frac{1}{|\mathcal{F}_k|} \sum_{t \in \mathcal{F}_k} \bm{\tilde{h}}_t,
\qquad
U_{t,k} = \cos\!\left(W_u\bm{\tilde{h}}_t, W_u \bm{r}_k\right).
\end{equation}
Here $W_u$ is a learnable projection matrix that maps senders and receivers into a shared utility space. Higher utility indicates that sending information along this path is more likely to preserve useful content for the compressed representation.

\textbf{Information capacity.} Different token anchors need not contribute equally to compression. To reflect their varying contextual importance, we assign each sender a learnable information capacity so that less useful anchors can transmit less mass while salient anchors retain more influence. We predict this capacity with a linear layer followed by a softmax over token anchors:
\begin{equation}
\rho_t = \operatorname{softmax}_{t}\!\left(W_\rho \bm{\tilde{h}}_t\right),
\qquad t = 1,\dots,N.
\end{equation}
Here $W_\rho$ is a learnable scoring matrix that produces the sender-capacity logits.

For the receivers, we use a uniform capacity $\rho_k = \frac{1}{K}$. This gives every compression slot the same total budget, while still allowing important anchors to connect to non-local slots to better preserve long-range dependencies.

\textbf{Transmission plan.} Given the utility matrix and the sender/receiver capacities, we derive the transmission plan by solving the following optimization problem:
\begin{equation}
\begin{aligned}
\min_{\Pi \ge 0} \quad 
& \sum_{t=1}^{N} \sum_{k=1}^{K} \Pi_{t,k} \, C_{t,k} \quad
\text{s.t.} \quad 
& \sum_{k=1}^{K} \Pi_{t,k} = \rho_t, \qquad \forall t, \
& \sum_{t=1}^{N} \Pi_{t,k} = \rho_k, \qquad \forall k,
\end{aligned}
\end{equation}
where the cost is defined as $C_{t,k} = 1 - U_{t,k}$. Unlike self-attention, whose weights are normalized independently for each slot, this objective couples all sender-slot assignments through shared marginal constraints. The resulting solution is exactly an optimal transport plan \citep{Villani2008OptimalTO} between senders and receivers, which directly targets the lack-of-allocation issue identified in the preliminary analysis. We solve the problem with entropy regularization, yielding a strictly convex objective that can be optimized efficiently with the Sinkhorn algorithm \citep{sinkhorn1967concerning, cuturi2013sinkhorn}. In practice, for stability and efficiency on long sequences, we run Sinkhorn on fixed-size segments of length $T$. We use a relatively large segment (e.g., $T=128$) to preserve broad allocation flexibility while avoiding overly aggressive long-range assignments that may disrupt local semantic order. This produces a globally coordinated yet locally aware transmission plan.

\textbf{Final representations.}
Given the transmission plan, each compression slot aggregates projected token-anchor features and is then aligned by a lightweight MLP:
\begin{equation}
\bm{z}_k = \mathrm{MLP}\!\left(\sum_{t=1}^{N} \Pi_{t,k} \, W_g \bm{\tilde{h}}_t\right), \qquad k=1,\ldots,K.
\end{equation}
Here $W_g$ is a learnable projection matrix used before slot-wise aggregation. The resulting compressed sequence is denoted by $Z=(\bm{z}_1,\ldots,\bm{z}_K)$.

\section{Experiments and Analysis}

\subsection{Experimental Setup}
\label{section:exp_details}
\textbf{Implementation details.}
In the experiment, we use Llama-3.2-1B-Base and Llama-3.2-3B-Base \citep{grattafiori2024llama3herdmodels} as the base model. For the projection to the decoder's input space, we set the projection hidden size to 256. The base models are trained using BF16 precision. We employ the Sinkhorn algorithm with 30 iterations to approximately solve the entropy-regularized optimal transport problem.

\textbf{Baselines.}
We compare with several soft state-of-the-art compression baselines: (1)~\textbf{ICAE}~\citep{ge2024incontext}: a seminal LLM-as-a-compressor method that trains an LLM to encode compressed information into compression tokens. (2)~\textbf{500×}~\citep{li-etal-2025-500xcompressor}: introduces an extra design that passes the KV states of the compression tokens to the decoder. (3)~\textbf{Activation Beacon}~\citep{zhang2025long}: introduces the interleaving placement of compression tokens. We also include three non-compression baselines: (4)~\textbf{Zero-shot [w/ context]} prompts the LLM with context, serving as an untrained baseline with zero information loss; (5)~\textbf{Zero-shot [w/o context]} prompts the LLM without the context, serving as an untrained baseline with complete information loss; and (6)~\textbf{Prompt tuning}~\citep{lester-etal-2021-power} also prepends learnable tokens while retaining full access to the context, and therefore serves as a strong uncompressed baseline in our experiments.

\textbf{Training and evaluation.}
Except for the base models evaluated under the inference-only setup, all baselines are trained under the same two-phase procedure: next-token prediction (NTP) on 1B tokens sampled from SlimPajama \citep{shen2023slimpajama}, followed by supervised fine-tuning (SFT) on MRQA \citep{fisch-etal-2019-mrqa} with six in-domain and six out-of-domain question-answering datasets. Only the context is compressed and each context can be used for more than one question. We set the compression ratio to $\times 4$.
All models use a context length of 512 tokens in both phases. We report EM and F1, and study additional settings in the ablations.

\subsection{Main Results}
\begin{table*}[!t]
\centering
\caption{Experimental results on six QA benchmark datasets. Best among compression methods is \textbf{bold}. Our method is colored in \colorbox{blue!6}{\quad}. The prompt-tuned uncompressed baseline is colored in \colorbox{gray!8}{\quad}.}
\label{tab:main}
\resizebox{\textwidth}{!}{%
\begin{tabular}{lcccccccccccccc}
\toprule
\multirow{2}{*}{\textbf{Methods}} &
\multicolumn{2}{c}{\textbf{SQuAD}} &
\multicolumn{2}{c}{\textbf{NewsQA}} &
\multicolumn{2}{c}{\textbf{TriviaQA}} &
\multicolumn{2}{c}{\textbf{SearchQA}} &
\multicolumn{2}{c}{\textbf{HotpotQA}} &
\multicolumn{2}{c}{\textbf{NQ}} &
\multicolumn{2}{c}{\textbf{Average}} \\
\cmidrule(lr){2-3}\cmidrule(lr){4-5}\cmidrule(lr){6-7}
\cmidrule(lr){8-9}\cmidrule(lr){10-11}\cmidrule(lr){12-13}\cmidrule(lr){14-15}
& \textbf{EM} & \textbf{F1}
& \textbf{EM} & \textbf{F1}
& \textbf{EM} & \textbf{F1}
& \textbf{EM} & \textbf{F1}
& \textbf{EM} & \textbf{F1}
& \textbf{EM} & \textbf{F1}
& \textbf{EM} & \textbf{F1} \\
\midrule

\multicolumn{15}{c}{\itshape Llama-3.2-1B} \\
\midrule

\rowcolor{gray!8}
Prompt tuning [w/ context]
& 71.89 & 81.09
& 30.72 & 45.76
& 61.04 & 66.75
& 64.03 & 70.64
& 53.60 & 69.49
& 53.03 & 66.59
& 55.72 & 66.72 \\

Zero-shot [w/ context]
& 16.59 & 38.76
& 9.97 & 26.99
& 35.50 & 49.31
& 5.83 & 12.39
& 22.51 & 39.24
& 18.14 & 38.90
& 18.09 & 34.27 \\

Zero-shot [w/o context]
& 2.80 & 8.83
& 0.97 & 3.11
& 15.63 & 21.60
& 4.04 & 6.64
& 2.76 & 7.17
& 3.37 & 8.43
& 4.93 & 9.29 \\

\midrule

ICAE \citep{ge2024incontext}
& 36.84 & 50.21
& 20.20 & 33.73
& 57.59 & 63.54
& 67.40 & 74.43
& 40.67 & 57.04
& 42.26 & 58.00
& 44.16 & 56.16 \\

500x \citep{li-etal-2025-500xcompressor}
& 4.65 & 14.02
& 1.85 & 5.27
& 26.09 & 32.66
& 17.44 & 25.59
& 5.56 & 13.44
& 5.58 & 13.89
& 10.20 & 17.48 \\

Beacon \citep{zhang2025long}
& 24.88 & 39.50
& 9.12 & 15.74
& 0.51 & 2.15
& 1.90 & 9.86
& 31.23 & 46.35
& 21.04 & 37.05
& 14.78 & 25.11 \\

\rowcolor{blue!6}
\method{}
& \textbf{51.38} & \textbf{68.08}
& \textbf{30.39} & \textbf{49.12}
& \textbf{64.62} & \textbf{70.93}
& \textbf{71.91} & \textbf{78.93}
& \textbf{49.84} & \textbf{68.40}
& \textbf{45.93} & \textbf{63.88}
& \textbf{52.34} & \textbf{66.55} \\

\midrule
\multicolumn{15}{c}{\itshape Llama-3.2-3B} \\
\midrule
\rowcolor{gray!8}
Prompt tuning [w/ context]
& 80.39 & 88.41
& 35.23 & 50.98
& 71.23 & 76.30
& 72.04 & 78.42
& 55.31 & 72.65
& 59.46 & 73.32
& 62.28 & 73.35 \\

Zero-shot [w/ context]
& 35.85 & 55.05
& 16.33 & 34.89
& 57.52 & 67.83
& 29.40 & 37.26
& 38.94 & 57.24
& 32.74 & 52.09
& 35.13 & 50.73 \\

Zero-shot [w/o context]
& 9.14 & 17.21
& 2.40 & 5.78
& 46.20 & 52.06
& 25.43 & 30.35
& 8.63 & 16.03
& 10.31 & 18.47
& 17.02 & 23.32 \\

\midrule
ICAE \citep{ge2024incontext}
& 48.51 & 62.53
& 27.02 & 43.04
& 68.62 & 74.13
& 76.43 & 82.89
& 47.92 & 66.16
& 50.87 & 66.75
& 53.23 & 65.92 \\

500x \citep{li-etal-2025-500xcompressor}
& 52.08 & 65.56
& 25.45 & 40.99
& 65.77 & 71.59
& 69.26 & 76.28
& 47.28 & 65.71
& 51.32 & 66.60
& 51.86 & 64.46 \\

Beacon \citep{zhang2025long}
& \textbf{59.81} & 72.65
& 15.46 & 24.39
& 1.36 & 4.72
& 5.24 & 9.13
& 51.55 & 67.41
& 45.05 & 59.56
& 29.75 & 39.64 \\

\rowcolor{blue!6}
\method{}
& 59.26 & \textbf{75.68}
& \textbf{35.73} & \textbf{55.23}
& \textbf{72.97 }& \textbf{78.86}
& \textbf{78.00} & \textbf{84.40}
& \textbf{55.92} & \textbf{74.15}
& \textbf{52.23} & \textbf{68.95}
& \textbf{59.02} & \textbf{72.88} \\

\bottomrule
\end{tabular}%
}
\end{table*}

\textbf{Overall performance.}
Under the question-unaware training setup described above, Table \ref{tab:main} shows that \method{} consistently outperforms all compression baselines on the six in-domain QA benchmarks for both Llama-3.2-1B and Llama-3.2-3B. In average F1, it improves over the strongest compression baseline by 18.50\% and 10.56\%, respectively, while remaining close to the uncompressed prompt-tuning baseline (within 0.25\% on 1B and 0.64\% on 3B). Overall, these results show that \method{} performs better under the same compression budget.

\textbf{Out-of-domain performance.}
Using the same training recipe and experimental setup, we further evaluate on six out-of-domain MRQA datasets in \cref{tab:mrqa_ood}. \method{} again performs best overall, improving average F1 over the strongest compression baseline by 30.97\% on Llama-3.2-1B and 16.00\% on Llama-3.2-3B, with only a small gap to ICAE on RelationExtraction. Since RelationExtraction has very short contexts (30 tokens on average) and ICAE and 500x employ a fixed budget of 128 compression tokens, they are constrained by a much lower compression ratio on this dataset, thereby yielding better performance. We also evaluate on summarization, fact-checking, and in-context learning in Appendix~\ref{appendix:exp-more-tasks}. \method{} outperforms the uncompressed baseline on XSum summarization and FEVER fact-checking while remaining competitive on SST-2 in-context learning. Together, these results suggest that \method{} generalizes well to various domains.

\subsection{Ablation Studies}
\definecolor{darkred}{RGB}{200,40,40}
\begin{table*}[t]
\centering
\begin{minipage}[t]{0.58\textwidth}
\footnotesize
\centering
\captionof{table}{Key design ablations for \method{} on Llama-3.2-1B (16 layers), using the main experiment's setup.}
\label{tab:ablation_design}
\renewcommand{\arraystretch}{1.08}
\setlength{\tabcolsep}{2pt}
\begin{tabular}{@{}p{0.47\linewidth}@{\hspace{1pt}}c@{\hspace{16pt}}c@{}}
\toprule
\textbf{Ablation} & \textbf{Avg. EM} & \textbf{Avg. F1} \\
\midrule
\method{} & 52.34 & 66.55 \\
w/o alloc. & 47.69 {\color{darkred}\scriptsize$\downarrow$4.65} & 61.94 {\color{darkred}\scriptsize$\downarrow$4.61}  \\
\midrule
only emb. layer & 28.65 {\color{darkred}\scriptsize$\downarrow$23.69} & 44.49 {\color{darkred}\scriptsize$\downarrow$22.06}  \\
only layer 5 & 45.13 {\color{darkred}\scriptsize$\downarrow$7.21} & 61.39 {\color{darkred}\scriptsize$\downarrow$5.16}  \\
only layer 9 & 49.34 {\color{darkred}\scriptsize$\downarrow$3.00} & 63.89 {\color{darkred}\scriptsize$\downarrow$2.66}  \\
only layer 14 & 44.18 {\color{darkred}\scriptsize$\downarrow$8.16} & 60.74 {\color{darkred}\scriptsize$\downarrow$5.81}  \\
only last layer & 36.12 {\color{darkred}\scriptsize$\downarrow$16.22} & 49.35 {\color{darkred}\scriptsize$\downarrow$17.20}  \\
\bottomrule
\end{tabular}
\end{minipage}\hfill
\begin{minipage}[t]{0.405\textwidth}
\footnotesize
\centering
\captionof{table}{Latency analysis results.}
\label{tab:latency}
\setlength{\tabcolsep}{3pt}
\resizebox{\linewidth}{!}{%
\begin{tabular}{@{}lccc@{}}
\toprule
\textbf{Method} & \textbf{Compress (s)} & \textbf{Decode (s)} & \textbf{E2E (s)} \\
\midrule
\multicolumn{4}{c}{\footnotesize\itshape\itshape Context Length = 512} \\
\midrule
Full Context & -- & 25.04 & 25.04 \\
ICAE & 0.43 & 10.16 & 10.59 \\
\method{} & \textbf{0.18} & \textbf{9.83} & \textbf{10.01} \\
\midrule
\multicolumn{4}{c}{\footnotesize\itshape\itshape Context Length = 2048} \\
\midrule
Full Context & -- & 101.09 & 101.09 \\
ICAE & 1.93 & 27.37 & 29.30 \\
\method{} & \textbf{0.84} & \textbf{27.18} & \textbf{28.02} \\
\bottomrule
\end{tabular}
}
\end{minipage}
\end{table*}
We examine our design choices under the same training and evaluation setup as before with Llama-3.2-1B. To isolate the effect of coordinated allocation, we replace the OT-based allocation module with a window attention mechanism that preserves locality but removes global coordination. As shown in \cref{tab:ablation_design}, performance drops consistently across all datasets, indicating that global allocation is critical for effective compression. Next, we remove layer-wise aggregation and use features from only a single LLM layer for compression. Performance degrades under all single-layer settings, with the middle layer (Layer 9) incurring the smallest drop. This result further validates that the design in \method{} that selectively aggregates features from all layers benefits compression.

We further vary three hyperparameters of the allocation module: the number of Sinkhorn iterations, the segment length $T$, and the projection dimension of the utility space. Performance saturates at around our default of 30 iterations and stays stable across projection dimensions and a larger segment, with clear drops only under very few iterations or a short segment. This shows that \method{} is robust to moderate changes of the OT hyperparameters (see Appendix~\ref{appendix:ot-sensitivity} for details).

\subsection{Latency Analysis}
\vspace{-0.5em}

Table \ref{tab:latency} reports the latency breakdown under two context lengths\footnote{Check Appendix~\ref{appendix:exp-latency} for details of the latency analysis experiment.}. Compared with ICAE, \method{} is more than $2\times$ faster overall in the compression stage across settings. In both settings, the compressed methods are also substantially faster end-to-end than decoding the full context, showing the efficiency gain brought by context compression. These results suggest that \method{} forms compressed representations more efficiently, yielding a stronger efficiency-quality trade-off. This compression time already includes the Sinkhorn step, which alone takes only about 5~ms per sample under our default setting and thus adds little overhead (Appendix~\ref{appendix:exp-latency}).

\vspace{-0.5em}
\subsection{Memory Analysis}
\vspace{-0.5em}
Since \method{} aggregates features across layers, it keeps the hidden states of multiple layers during compression, which leads to a higher peak GPU memory than ICAE at long contexts (32.57 GB vs.~23.44 GB at 32k tokens). This overhead can be reduced by windowed compression, a common strategy for long contexts that compresses the context window by window and keeps only the hidden states of the current window. It lowers the peak memory of \method{} to 16.83 GB, below that of ICAE, with the same compression time. Since we only measure the efficiency of this windowed variant, we leave training it to future work (Appendix~\ref{appendix:long-memory}).

\vspace{-0.5em}
\subsection{Scalability Analysis}
\vspace{-0.5em}
\begin{wraptable}{r}{0.45\columnwidth}
\footnotesize
\centering
\caption{Scalability analysis experiments for \method{} on higher compression ratios, larger models, and longer context.}
\label{tab:ablation_other}
\renewcommand{\arraystretch}{1.08}
\setlength{\tabcolsep}{2pt}
\resizebox{\linewidth}{!}{%
\begin{tabular}{@{}p{0.29\linewidth}@{\hspace{4pt}}p{0.45\linewidth}@{\hspace{6pt}}cc@{}}
\toprule
\textbf{Analysis} & \textbf{Setting} & \textbf{Avg. EM} & \textbf{Avg. F1} \\
\midrule
\multirow{4}{*}{\shortstack{Compression\\ratio}} & $\times 8$ - ICAE & 38.72 & 50.07 \\
& $\times 8$ - \method{} & \textbf{44.83} & \textbf{58.66} \\
& $\times 16$ - ICAE & 37.23 & 48.82 \\
& $\times 16$ - \method{} & \textbf{40.01} & \textbf{52.53} \\
\midrule
\multirow{3}{*}{Model size} & 8B - Prompt tuning & \textbf{65.03} & 75.92 \\
& 8B - ICAE & 59.24 & 70.41 \\
& 8B - \method{} & 63.25 & \textbf{76.48} \\
\midrule
\multirow{3}{*}{Context length} & 8k - Prompt tuning & 27.43 & 37.27 \\
& 8k - ICAE & 21.99 & 30.00 \\
& 8k - \method{} & \textbf{32.46} & \textbf{43.97} \\
\bottomrule
\end{tabular}%
}
\end{wraptable}
We examine the scalability of our method from three dimensions: higher compression ratios, larger models and longer context.

\textbf{Higher compression ratios.}
Under the same setup, we train the Llama-3.2-1B with compression ratios of $\times 8$ and $\times 16$ and report the results in \cref{tab:ablation_other}. \method{} outperforms the baseline at both $8\times$ (+17.2\% Avg. F1) and $16\times$ (+7.6\%). Notably, \method{} at $8\times$ compression (58.66) surpasses ICAE at $4\times$ compression (56.16; \cref{tab:main}). These results show that \method{} exhibits consistently strong performance under higher compression ratios.

\textbf{Larger model sizes.}
To examine the scalability of \method{} on larger models, we train the methods with Llama-3.1-8B under the same setup and report the performance in \cref{tab:ablation_other}. \method{} continues to outperform ICAE by a clear margin and even slightly surpasses the uncompressed prompt-tuning baseline in F1 (76.48 vs.~75.92), showing that its advantage persists at larger model scale.

\textbf{Longer context.}
We train the Llama-3.2-1B model on three long-context QA datasets, TriviaQA~\citep{joshi-etal-2017-triviaqa}, QuALITY~\citep{pang-etal-2022-quality}, and NaturalQuestions~\citep{kwiatkowski-etal-2019-natural}, with contexts truncated to a maximum length of 8k. The results are shown in \cref{tab:ablation_other}. On this setting, \method{} improves over the uncompressed prompt-tuning baseline by 18.0\% in F1 and outperforms ICAE by a large margin. This suggests that \method{} can preserve task-relevant content under more challenging long-context compression settings.

\vspace{-0.5em}
\subsection{Further Analysis}
\vspace{-0.5em}

\subsubsection{Information Allocation}
\vspace{-0.5em}

\begin{figure*}[tb]
     \centering
    \begin{subfigure}[b]{0.6\columnwidth}
    \centering
    \begin{subfigure}[b]{0.48\linewidth}
        \centering
        \includegraphics[width=\linewidth]{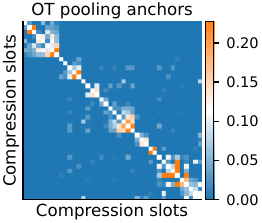}
        \caption{Pearson correlation.}
        \label{fig:pearson}
    \end{subfigure}\hfill
    \begin{subfigure}[b]{0.48\linewidth}
        \centering
        \includegraphics[width=\linewidth]{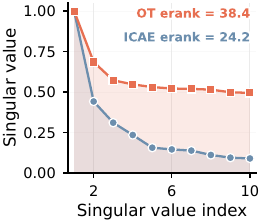}
        \caption{Singular-value spectrum.}
        \label{fig:singulaValue}
    \end{subfigure}
    \label{fig:allocation_analysis}
\end{subfigure}
\hfil
\begin{subfigure}[b]{0.33\columnwidth}
    \centering
    \includegraphics[width=\linewidth]{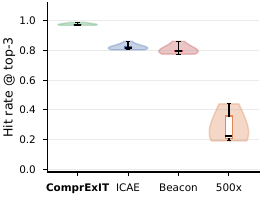}
    \caption{Entity coverage.}
    \label{fig:entity_coverage}
\end{subfigure}
   \caption{Information allocation analysis. \textbf{(a)} Pearson correlation (diagonal masked) and \textbf{(b)} singular-value spectrum of allocation patterns. \textbf{(c)} Question-relevant entity recall on HotpotQA by TF-IDF rank. Top-10 entities are selected per context, and violin plot shows retrieval consistency across examples.}
   \label{fig:info_allocation}
\end{figure*}

\textbf{Allocation patterns.} To further understand the impact cast by coordinated allocation, we analyze the aggregation behaviors of compression tokens\footnote{Experiment specifications for examining the allocation patterns can be found in Appendix~\ref{appendix:exp-alloc}.}. In \Cref{fig:pearson}, compared with the aggregation pattern of ICAE in \cref{fig:observations_allocation}, the consistently lower inter-slot correlations of \method{} indicate that different compression tokens absorb complementary information rather than collapsing into redundancy. In \Cref{fig:singulaValue}, this same behavior is reflected in the singular-value spectrum: compared with ICAE, \method{} exhibits a slower decay and a higher effective rank, showing that its allocation patterns span a richer and less redundant subspace. 
Taken together, the two views suggest that without global allocation, compression tokens tend to become correlated and low-rank, whereas \method{} encourages a more diverse and information-efficient distribution of content across slots.

\textbf{Coverage of entities.} We further examine whether coordinated allocation helps preserve important entities during compression.
We extract named entities from each context and rank them by TF-IDF score. For each gist token, we identify the top-3 entities it attends to most and take their union within the context. We then compute the coverage of the ten most question-relevant entities by the gist tokens during compression, as shown in \Cref{fig:entity_coverage}. Note that TF-IDF is used here only as a selection heuristic, and its notion of relevance is limited, since entities that are less relevant to one question may still be important for others. \method{} achieves the highest average coverage rate, nearly 1.0. More importantly, its thin violin shape indicates that it consistently covers relevant entities across examples, whereas the larger variance of the baselines suggests that they allocate attention more unevenly and sometimes miss important entities in certain contexts. These results indicate that a well-coordinated allocation is better able to preserve critical contextual information during compression.

\vspace{-0.5em}
\subsubsection{Layer Selection}
\vspace{-0.5em}
\begin{figure}[!h]
    \centering
    \begin{subfigure}[t]{0.49\columnwidth}
        \centering
        \includegraphics[width=\linewidth]{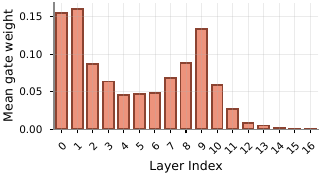}
    \end{subfigure}\hfill
    \begin{subfigure}[t]{0.49\columnwidth}
        \centering
        \includegraphics[width=\linewidth]{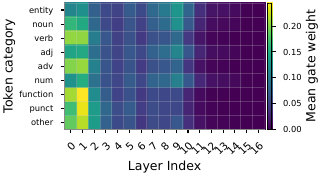}
    \end{subfigure}
    \caption{Layer-selection analysis in \method{}. \textit{Left}: depth-wise gating weights across layers. \textit{Right}: position--layer heatmap showing token-wise layer preferences.}
    \label{fig:depth-gate}
\end{figure}
\textbf{General layer preference.} Figure \ref{fig:depth-gate}-left shows that the gating mass concentrates on early and middle layers, with a clear emphasis before layer 10 and consistently low weights on later layers. This suggests that effective compression mainly relies on representations that still preserve lexical and contextual information, while deeper layers are more specialized for generation and therefore less suitable for forming compact, decoder-friendly representations.

\textbf{Token-wise layer preference.} The right panel shows a clear dependence on token type. Entity tokens, common nouns, adjectives, and numbers tend to receive higher weights from middle layers, where contextual, relational, and attribute-level information is more salient. In contrast, other tokens are more often routed to the initial layers, indicating that shallow lexical features are usually sufficient for them. Overall, these patterns suggest that \method{} adapts layer selection to token type, reserving richer mid-layer features for informative content words and using earlier-layer features for less informative tokens.

\vspace{-0.5em}
\subsubsection{Dependence on the Decoder}
\label{sec:probe}
\vspace{-0.5em}
\begin{wraptable}{r}{0.36\textwidth}
\vspace{-1em}
\footnotesize
\centering
\caption{Context reconstruction by a decoder trained from scratch on the frozen compressed representations (Llama-3.2-3B, $4\times$ compression).}
\label{tab:probe}
\setlength{\tabcolsep}{5pt}
\begin{tabular}{@{}lcc@{}}
\toprule
\textbf{Method} & \textbf{NLL} $\downarrow$ & $\bm{H_r}$ $\uparrow$ \\
\midrule
ICAE & 4.294 & 0.635 \\
500x & 4.176 & 0.645 \\
\rowcolor{blue!6}
\method{} & \textbf{0.639} & \textbf{0.946} \\
\bottomrule
\end{tabular}
\vspace{-1em}
\end{wraptable}
Since the compressors are trained to produce inputs for the decoder LLM, one may question whether a higher F1 comes from compressed representations that are easier for this particular decoder to read, rather than from representations that retain more information about the context. To answer this question, we follow \citet{badger2026language} and probe the compressed representations with a decoder that is \emph{independent of the decoder LLM}. Specifically, we keep each trained compressor frozen and train a randomly initialized 8-layer decoder (158M parameters) from scratch to reconstruct the context from its compressed representations\footnote{Experimental details are in Appendix~\ref{appendix:exp-probe}.}. We report the reconstruction NLL and the entropy ratio $H_r$~\citep{badger2026language} on held-out contexts, where $H_r=1$ means the context is fully recovered and $H_r=0$ means the prediction is no better than a uniform guess over the vocabulary. As shown in \cref{tab:probe}, the probing decoder recovers much more of the context from the representations of \method{} than from those of ICAE and 500x. We take this as evidence that the compressed representations of \method{} preserve more information about the context, and that its advantage does not come only from representations that better suit the decoder LLM.


\vspace{-0.5em}
\section{Conclusion}
\label{sec:conclusion}
\vspace{-0.5em}
This work presents \method{}, a new paradigm for soft context compression that formulates compression as explicit information transmission over frozen LLM hidden states, mitigating the limitations of layerwise dilution and lack of allocation. Experiments show that it consistently outperforms prior compression methods. More broadly, the proposed paradigm provides a more flexible design space beyond standard attention-based formulations. Therefore, we hope that this work can inspire future research to explore more design choices in this new paradigm and further advance context compression research.

\textbf{Limitations.} First, due to limited computational resources, we do not test \method{} on larger backbones (e.g., 32B or 70B models), and we evaluate task performance only up to 8k-token contexts, measuring efficiency alone at 16k and 32k. Second, \method{} keeps the hidden states of multiple layers during compression, so its peak memory is higher than that of ICAE unless the context is compressed window by window. Third, our experiments focus on text-only language models. Extending the framework to vision-language models, which also accept tokenized inputs, is left to future work.

\newpage
\begin{ack}
This work was supported in part by the UK Engineering and Physical Sciences Research Council through a Turing AI Fellowship (grant no. EP/V020579/1, EP/V020579/2) and the Prosperity Partnership scheme (grant no. UKRI566). 

The authors acknowledge use of King's Computational Research, Engineering and Technology Environment (CREATE) at King's College London \citep{kcl2026create}.

The authors acknowledge the use of resources provided by the Dawn National AI Research Resource (AIRR). Dawn is operated by the University of Cambridge and is funded by the UK Government’s Department for Science, Innovation and Technology (DSIT) via UK Research and Innovation the Science and Technology Facilities Council [ST/Z000890/1], Dell Technologies and Intel.

The authors declare no competing interests.
\end{ack}

\bibliographystyle{plainnat}
\bibliography{sections/reference}

\appendix
\clearpage
\section{Appendix}
\subsection{Experiment Specification}
We present more details of the experiments conducted in our paper.

\subsubsection{Measuring Information Loss during Compression}
\label{appendix:exp-recons}
For each method and each compressor, we extract the gist hidden
states at layer $\ell$, apply the method's own final projection where one exists, and feed them to the frozen decoder as
the prefix of the sequence $[\textsc{bos}] \,\Vert\, \mathbf{g}^{(\ell)} \,\Vert\, \mathrm{Embed}(\mathbf{p}) \,\Vert\, \mathrm{Embed}(\mathbf{x})$,
where $\mathbf{p}$ is the fixed prompt
\texttt{"Please reproduce the previous passage exactly."} and $\mathbf{x}$
is the original context. We calculate the standard causal-LM cross-entropy loss on the
context tokens only. We then aggregate token-weighted across the dataset,
$\widehat{\mathcal{L}}_\ell = \sum_i \mathrm{loss}_i \,/\, \sum_i n_i$, where
$n_i$ is the number of non-padded target tokens in example $i$.
We compute this loss on each of the six in-domain datasets (Appendix~\ref{app:data}) and report its average over these datasets in \cref{fig:intro}, together with the average F1 on the same datasets. We use the compression models trained with Llama-3.2-3B in this experiment. 

\subsubsection{Latency Analysis}
\label{appendix:exp-latency}
We benchmark each method on synthetic batches with fixed context lengths of 512 and 2048, a batch size of 8, a generation length of 128 tokens, and a compression ratio of 4. We use NVIDIA L40S ADA GPU, with 48GB memory. For each sequence, we run $50$ untimed warm-up iterations followed by $200$ timed iterations, and report the mean. We separately time the \emph{compression} stage (the encoder forward pass that produces the gist) and the \emph{decoding} stage (autoregressive generation of $128$ tokens), and report both along with their sum (E2E).

\begin{wraptable}{r}{0.40\textwidth}
\vspace{-1.2em}
\footnotesize
\centering
\caption{Per-sample time (ms) of the Sinkhorn step with a batch size of 64. Our default setting is \underline{underlined}.}
\label{tab:ot_time}
\setlength{\tabcolsep}{4pt}
\begin{tabular}{@{}lcccc@{}}
\toprule
& \multicolumn{4}{c}{\textbf{Sinkhorn iterations}} \\
\cmidrule(lr){2-5}
\textbf{Segment $T$} & 10 & 30 & 50 & 100 \\
\midrule
64 & 4.83 & 5.18 & 5.52 & 6.37 \\
128 & 4.81 & \underline{5.45} & 6.10 & 7.69 \\
512 & 5.46 & 7.28 & 9.12 & 13.65 \\
\bottomrule
\end{tabular}
\vspace{-1em}
\end{wraptable}
\textbf{Latency of the Sinkhorn step.} The compression time above already includes the Sinkhorn step of the allocation module. To further measure its cost, we time the Sinkhorn step alone under different segment lengths $T$ and numbers of iterations, using Llama-3.2-1B with a context length of 512 and a batch size of 64 on a single NVIDIA A100 80GB PCIe GPU. As shown in \cref{tab:ot_time}, it takes about 5~ms per sample under our default setting ($T=128$ and 30 iterations). The time grows only mildly with more iterations or a larger segment, and stays below 14~ms even with $T=512$ and 100 iterations, since the Sinkhorn step can be computed in parallel.

\subsubsection{Gist Token Allocation Patterns}
\label{appendix:exp-alloc}
We probe whether the $K$ compressed slots of each method actually carry
distinct content, or whether they collapse onto similar input regions. We run our experiments on the SQuAD~\citep{rajpurkar-etal-2016-squad} dataset with 1000 random samples. For ICAE we take the compressor's last layer attention map of the gist tokens and average over the attention heads. For \method{}, we use the captured transport plan as the allocation scores. In both cases this yields a $[K, N]$ matrix $A$ whose rows are the per-slot distributions
over input tokens. The \emph{similarity heatmap} is the $K \times K$
Pearson-correlation matrix between these rows,
$\rho(a_i, a_j) = (a_i - \bar a_i)^\top (a_j - \bar a_j) / (\|a_i - \bar a_i\|_2 \|a_j - \bar a_j\|_2)$,
reordered hierarchically for readability. Its diagonal entries are always 1 and carry no information, so we mask them (shown in white) to keep them from setting the top of the color scale, which would make the lower off-diagonal correlations hard to tell apart. The
\emph{attention spectrum} is the sequence of singular values of $A$, plus
the effective rank
$\mathrm{erank}(A) = \exp\!\big(-\!\sum_i \tilde{\sigma}_i \log \tilde{\sigma}_i\big)$
with $\tilde{\sigma}_i = \sigma_i / \sum_j \sigma_j$. A high
effective rank and a near-diagonal similarity heatmap indicate that the
$K$ compressed slots cover the input non-redundantly.

\subsubsection{Probing with a Decoder Trained from Scratch}
\label{appendix:exp-probe}
We follow the probing setup of \citet{badger2026language}. For each method, we take the compressor trained with Llama-3.2-3B at a compression ratio of $4\times$ in the main experiment, keep it frozen, and use the last-layer hidden states of its compression tokens as the compressed representations. For 500x, we also use these hidden states rather than the KV states that it passes to the decoder LLM, so that all methods are probed in the same way. The compressed representations are mapped to a hidden size of 768 by a single linear projection, and are then fed as a prefix to a randomly initialized 8-layer Transformer decoder with 158M parameters. The decoder is trained to reconstruct the context with the next-token prediction loss on the context tokens. We use contexts of 512 tokens from the training splits of the six in-domain datasets (Appendix~\ref{app:data}), train the decoder for 10k steps, and evaluate it on the validation splits of the same datasets. In addition to the reconstruction NLL, we report the entropy ratio
\begin{equation*}
H_r = 1 - \frac{\mathrm{NLL}}{\log |\mathcal{V}|},
\end{equation*}
where $\log |\mathcal{V}|$ is the cross-entropy of a uniform prediction over the vocabulary $\mathcal{V}$ of the Llama-3 tokenizer ($|\mathcal{V}|=128{,}256$). In other words, $H_r$ measures how much the NLL is reduced relative to this uniform baseline, so a higher $H_r$ indicates that more information about the context can be recovered from the compressed representations.

\subsection{Datasets}
\label{app:data}
\cref{tab:dataset_stats} shows the statistics of the datasets we use in the experiments.

\begin{table}[htbp]
\caption{Statistics of the training and evaluation datasets used, including in-domain and out-of-domain datasets. \#Train represents the number of training samples and \#Validation represents the number of validation samples.}
\centering
\small
\setlength{\tabcolsep}{10pt}
\renewcommand{\arraystretch}{1.15}

\resizebox{\linewidth}{!}{%
\begin{tabular}{lcccc}
\toprule
\textbf{Dataset} & \textbf{Average Question Length} & \textbf{Average Context Length} & \textbf{\#Train} & \textbf{\#Validation} \\
\midrule
SQuAD              & 11 & 137 & 86,588  & 10,507 \\
NewsQA             &  8 & 599 & 74,160  &  4,212 \\
TriviaQA           & 16 & 784 & 61,688  &  7,785 \\
SearchQA           & 17 & 749 & 117,384 & 16,980 \\
HotpotQA           & 22 & 232 & 72,928  &  5,904 \\
Natural Questions  &  9 & 153 & 104,071 & 12,836 \\
\midrule
BioASQ             & 11 & 248 & --      & 1,518 \\
DROP               & 11 & 243 & --      & 1,501 \\
DuoRC              &  9 & 681 & --      & 1,503 \\
RACE               & 12 & 349 & --      & 1,502 \\
RelationExtraction &  9 &  30 & --      & 1,500 \\
TextbookQA         & 11 & 657 & --      & 1,508 \\
\bottomrule
\end{tabular}%
}
\label{tab:dataset_stats}
\end{table}

\subsection{Model Inference Implementation}
We carefully handle the positions of the padding tokens, which we find brings improvements to all methods. We keep the no-padding setup during NTP to make full use of the GPUs. In SFT, as we stick to the question-unaware setup, we have to encode the contexts first and feed the context compression tokens along with the question tokens to the decoder. For efficient batch compression, we first apply left-padding to the context tokens for the compressors. We then concatenate the compression tokens with padded question and answer tokens, and perform an operation to move all the padding tokens to the leftmost side according to the attention mask. We also create the shifted new attention mask and labels accordingly. This ensures that the generation starts from a non-padding token and there is no position gap between the compression tokens and the questions.

\subsection{Evaluation on More Tasks}
\label{appendix:exp-more-tasks}
We further evaluate \method{} on summarization, fact-checking, and in-context learning. For summarization, we fine-tune all methods on XSum~\citep{narayan-etal-2018-dont} with Llama-3.2-1B and report the results in \cref{tab:more_tasks_sum}. \method{} achieves the best ROUGE-1/ROUGE-L scores, 0.333/0.262, outperforming both ICAE and the uncompressed prompt-tuning baseline. This suggests that \method{} preserves not only answer-relevant facts but also the broader discourse information needed for abstractive generation. For fact-checking and in-context learning, we directly evaluate the NTP-trained Llama-3.2-3B models on FEVER~\citep{thorne-etal-2018-fever} and SST-2~\citep{socher-etal-2013-recursive}, respectively, as shown in \cref{tab:more_tasks_fever_icl}. On FEVER, \method{} achieves the best score, 64.39, surpassing the prompt-tuning baseline by 3.91\% and ICAE by 8.69\%. On SST-2 ICL, \method{} obtains 69.17, improving over ICAE by 4.79\% and remaining competitive with the prompt-tuning baseline (72.34). Overall, these results show that \method{} generalizes consistently well beyond QA.

\begin{table*}[t]
\small
\centering
\begin{minipage}[t]{0.48\textwidth}
\centering
\captionof{table}{ROUGE-1 and ROUGE-L scores on the summarization (XSum) dataset.}
\label{tab:more_tasks_sum}
\begin{tabular*}{\linewidth}{@{\extracolsep{\fill}}lcc@{}}
\toprule
\textbf{Model} & \textbf{ROUGE-1} & \textbf{ROUGE-L} \\
\midrule
ICAE & 0.324 & 0.257 \\
Prompt Tuning & 0.320 & 0.253 \\
\textbf{\method{}} & \textbf{0.333} & \textbf{0.262} \\
\bottomrule
\end{tabular*}
\end{minipage}\hfill
\begin{minipage}[t]{0.48\textwidth}
\centering
\captionof{table}{F1 scores on fact checking (FEVER) and in-context learning (SST-2) datasets.}
\label{tab:more_tasks_fever_icl}
\begin{tabular*}{\linewidth}{@{\extracolsep{\fill}}lcc@{}}
\toprule
\textbf{Model} & \textbf{FEVER} & \textbf{SST-2 ICL} \\
\midrule
ICAE & 59.24 & 66.01 \\
Prompt Tuning & 61.97 & \textbf{72.34} \\
\textbf{\method{}} & \textbf{64.39} & 69.17 \\
\bottomrule
\end{tabular*}
\end{minipage}
\end{table*}

\subsection{Out-of-Domain Results}
Table~\ref{tab:mrqa_ood} reports the full results on the six out-of-domain MRQA benchmarks. Overall, \method{} remains the strongest compression method across both Llama-3.2-1B-Base and Llama-3.2-3B-Base, showing that its learned compressed representations transfer well beyond the training domains. We note that ICAE is competitive on RelationExtraction, whose contexts are very short on average. Because prior methods use a fixed budget of 128 compression tokens, this dataset effectively gives them a much smaller compression ratio, which partially explains the smaller gap on that benchmark.

\begin{table*}[!t]
\centering
\caption{Experimental results on six \textbf{out-of-domain} MRQA benchmark datasets. Out-of-domain: these datasets are not included in the training data. Best among compression methods is \textbf{bold}. Our method is colored in \colorbox{blue!6}{\quad}. The prompt-tuned uncompressed baseline is colored in \colorbox{gray!8}{\quad}. Note that ICAE and 500x have a smaller compression ratio on samples with contexts shorter than 512 tokens because they are trained with a fixed set of 128 compression tokens. In particular, on the RelationExtraction dataset, whose average context length is only 30, they use more compression tokens than context tokens.} 

\label{tab:mrqa_ood}
\resizebox{\textwidth}{!}{%
\begin{tabular}{lcccccccccccccc}
\toprule
\multirow{2}{*}{Methods} &
\multicolumn{2}{c}{RelationExtraction} &
\multicolumn{2}{c}{BioASQ} &
\multicolumn{2}{c}{TextbookQA} &
\multicolumn{2}{c}{DuoRC} &
\multicolumn{2}{c}{DROP} &
\multicolumn{2}{c}{RACE} &
\multicolumn{2}{c}{Average} \\
\cmidrule(lr){2-3}\cmidrule(lr){4-5}\cmidrule(lr){6-7}
\cmidrule(lr){8-9}\cmidrule(lr){10-11}\cmidrule(lr){12-13}\cmidrule(lr){14-15}
& EM & F1
& EM & F1
& EM & F1
& EM & F1
& EM & F1
& EM & F1
& EM & F1 \\
\midrule

\multicolumn{15}{c}{\itshape Llama-3.2-1B-Base} \\
\midrule
\rowcolor{gray!8}
Prompt tuning [w/ context]
& 69.91 & 81.15
& 65.96 & 78.80
& 52.10 & 59.61
& 38.51 & 47.40
& 32.54 & 43.14
& 33.68 & 46.30
& 48.78 & 59.40 \\

Zero-shot [w/ context]
& 27.48 & 47.87
& 6.98 & 28.84
& 18.16 & 29.97
& 17.46 & 30.54
& 18.43 & 32.30
& 10.53 & 22.71
& 16.51 & 32.04 \\

Zero-shot [w/o context]
& 5.09 & 10.33
& 14.76 & 20.56
& 11.11 & 16.03
& 0.60 & 2.55
& 14.84 & 19.73
& 0.45 & 2.88
& 7.81 & 12.01 \\

\midrule
ICAE \citep{ge2024incontext}
& \textbf{62.28} & \textbf{75.97}
& 42.62 & 54.09
& 28.21 & 34.60
& 12.59 & 18.80
& 24.15 & 33.23
& 9.50 & 16.95
& 29.89 & 38.94 \\

500x \citep{li-etal-2025-500xcompressor}
& 6.31 & 16.30
& 22.54 & 35.52
& 16.77 & 24.40
& 1.07 & 3.91
& 8.18 & 18.32
& 1.93 & 5.70
& 9.47 & 17.36 \\

Beacon \citep{zhang2025long}
& 33.18 & 48.82
& 32.25 & 47.84
& 5.92 & 10.34
& 2.40 & 4.90
& 21.82 & 29.24
& 15.88 & 25.35
& 18.57 & 27.75 \\

\rowcolor{blue!6}
\method{}
& 59.46 & 74.96
& \textbf{48.40} & \textbf{63.74}
& \textbf{41.65} & \textbf{50.73}
& \textbf{28.78} & \textbf{38.77}
& \textbf{29.01} & \textbf{39.86}
& \textbf{23.00} & \textbf{37.95}
& \textbf{38.38} & \textbf{51.00} \\
\midrule
\multicolumn{15}{c}{\itshape Llama-3.2-3B-Base} \\
\midrule

\rowcolor{gray!8}
Prompt tuning [w/ context]
& 77.37 & 86.70
& 67.62 & 81.50
& 60.48 & 68.75
& 45.50 & 55.63
& 48.84 & 57.18
& 43.62 & 56.35
& 57.24 & 67.68 \\

Zero-shot [w/ context]
& 54.24 & 68.46
& 43.22 & 64.31
& 36.46 & 46.66
& 28.31 & 42.36
& 23.69 & 37.45
& 22.70 & 38.16
& 34.77 & 49.57 \\

Zero-shot [w/o context]
& 14.28 & 21.03
& 20.94 & 25.76
& 23.29 & 28.89
& 2.07 & 4.44
& 16.17 & 21.26
& 1.78 & 5.65
& 13.09 & 17.84 \\

\midrule
ICAE \citep{ge2024incontext}
& \textbf{70.66} & \textbf{82.10}
& 55.85 & 67.21
& 42.25 & 50.70
& 23.65 & 32.81
& 35.33 & 44.89
& 21.51 & 33.85
& 41.54 & 51.93 \\

500x \citep{li-etal-2025-500xcompressor}
& 65.13 & 79.67
& 47.67 & 60.82
& 29.41 & 38.10
& 22.25 & 30.91
& 29.74 & 39.12
& 19.14 & 30.58
& 35.56 & 46.53 \\

Beacon \citep{zhang2025long}
& 58.92 & 71.62
& 54.52 & 66.86
& 10.98 & 15.68
& 4.73 & 7.95
& 39.39 & 50.55
& 30.71 & 43.95
& 33.21 & 42.77 \\

\rowcolor{blue!6}
\method{}
& 65.31 & 77.25
& \textbf{59.64} & \textbf{73.87}
& \textbf{56.09} & \textbf{64.50}
& \textbf{34.31} & \textbf{44.98}
& \textbf{43.58} & \textbf{53.16}
& \textbf{31.60} & \textbf{47.69}
& \textbf{48.42} & \textbf{60.24} \\

\bottomrule
\end{tabular}%
}
\end{table*}

\subsection{Sensitivity to OT Hyperparameters}
\label{appendix:ot-sensitivity}
\begin{wraptable}{r}{0.44\textwidth}
\vspace{-1em}
\footnotesize
\centering
\caption{Sensitivity of \method{} to OT hyperparameters on Llama-3.2-1B (Avg. F1 on the six in-domain datasets). Our default settings are \underline{underlined}.}
\label{tab:sensitivity}
\setlength{\tabcolsep}{3pt}
\begin{tabular}{@{}lccccc@{}}
\toprule
\textbf{Sinkhorn iter.} & 5 & 10 & \underline{30} & 50 & 100 \\
Avg. F1 & 61.32 & 65.49 & 66.55 & 66.58 & 66.58 \\
\midrule
\textbf{Segment $T$} & 64 & \underline{128} & 512 & & \\
Avg. F1 & 64.07 & 66.55 & 66.00 & & \\
\midrule
\textbf{Proj. dim.} & 64 & \underline{256} & 512 & 1024 & \\
Avg. F1 & 65.64 & 66.55 & 65.89 & 65.62 & \\
\bottomrule
\end{tabular}
\vspace{-1em}
\end{wraptable}
We study how sensitive \method{} is to three hyperparameters of the allocation module: the number of Sinkhorn iterations, the segment length $T$ on which Sinkhorn is run, and the projection dimension of the utility space. For each setting, we train Llama-3.2-1B with the full NTP and SFT procedure under the same setup as the main experiment. As shown in \cref{tab:sensitivity}, performance saturates at around 30 iterations, our default setting, and more iterations bring only a negligible gain (66.58 at both 50 and 100 iterations). Even with 5 iterations, \method{} still reaches 61.32 and outperforms the strongest compression baseline in \cref{tab:main} (56.16), since a few iterations are already enough to produce a rough allocation plan. For the segment length, a small segment ($T=64$) degrades performance because it limits long-range allocation from token anchors to compression slots, while a larger segment ($T=512$) brings no further gain, which supports our choice of $T=128$ in \cref{sec:method}. The projection dimension has little effect, and all settings stay within one F1 point of each other. Overall, these results show that our default settings lie in a stable range, where moderate changes of the hyperparameters have little effect on performance. Note that we do not tune the number of compression slots separately, since it is determined by the compression ratio, which we study in \cref{tab:ablation_other}.

\subsection{Memory Cost under Long Contexts}
\label{appendix:long-memory}
To study the memory cost of \method{} at long contexts, we extend the latency analysis (Appendix~\ref{appendix:exp-latency}) to context lengths of 16k and 32k. We use the models trained with Llama-3.2-3B in the main experiment, and benchmark each method in BF16 precision with FlashAttention-2, a batch size of 1, and a generation length of 4096 tokens on a single NVIDIA A100 80GB PCIe GPU. We report the end-to-end peak GPU memory together with the compression time. We do not include 128k contexts, since they do not fit into the memory of a single 80GB GPU. For windowed compression, we split the context into windows of 2048 tokens and compress them sequentially, so that only the hidden states of the current window need to be kept in memory. Note that these windows are different from the segments of length $T$ on which Sinkhorn is run.

\begin{wraptable}{r}{0.50\textwidth}
\vspace{-1.2em}
\footnotesize
\centering
\caption{Peak GPU memory (Mem., GB) and compression time (Time, s) at long contexts. The windowed variant is only used to measure efficiency and is not trained.}
\label{tab:long_latency}
\setlength{\tabcolsep}{5pt}
\begin{tabular}{@{}lcccc@{}}
\toprule
\multirow{2}{*}{\textbf{Method}} & \multicolumn{2}{c}{\textbf{16k}} & \multicolumn{2}{c}{\textbf{32k}} \\
\cmidrule(lr){2-3}\cmidrule(lr){4-5}
& \textbf{Mem.} & \textbf{Time} & \textbf{Mem.} & \textbf{Time} \\
\midrule
Full Context & 8.91 & -- & 11.79 & -- \\
ICAE & 15.47 & 1.12 & 23.44 & 2.75 \\
\method{} & 19.33 & \textbf{0.52} & 32.57 & \textbf{1.36} \\
\method{} (windowed) & \textbf{11.76} & \textbf{0.52} & \textbf{16.83} & \textbf{1.36} \\
\bottomrule
\end{tabular}
\vspace{-1em}
\end{wraptable}
As shown in \cref{tab:long_latency}, \method{} has a higher peak memory than ICAE at both lengths, because it needs to keep the hidden states of multiple layers for the layer-wise aggregation. At the same time, its compression stage remains about $2\times$ faster than that of ICAE. Windowed compression effectively reduces this memory overhead: it lowers the peak memory of \method{} from 19.33 GB to 11.76 GB at 16k and from 32.57 GB to 16.83 GB at 32k, both below that of ICAE, while keeping the same compression time. Since we only measure the efficiency of this windowed variant, we leave training \method{} with windowed compression to future work.

\subsection{The Allocation Matrix}
\begin{figure}[htbp] 
    \centering
    \subfloat[ICAE]{%
        \includegraphics[width=0.49\linewidth,height=0.28\textheight,keepaspectratio]{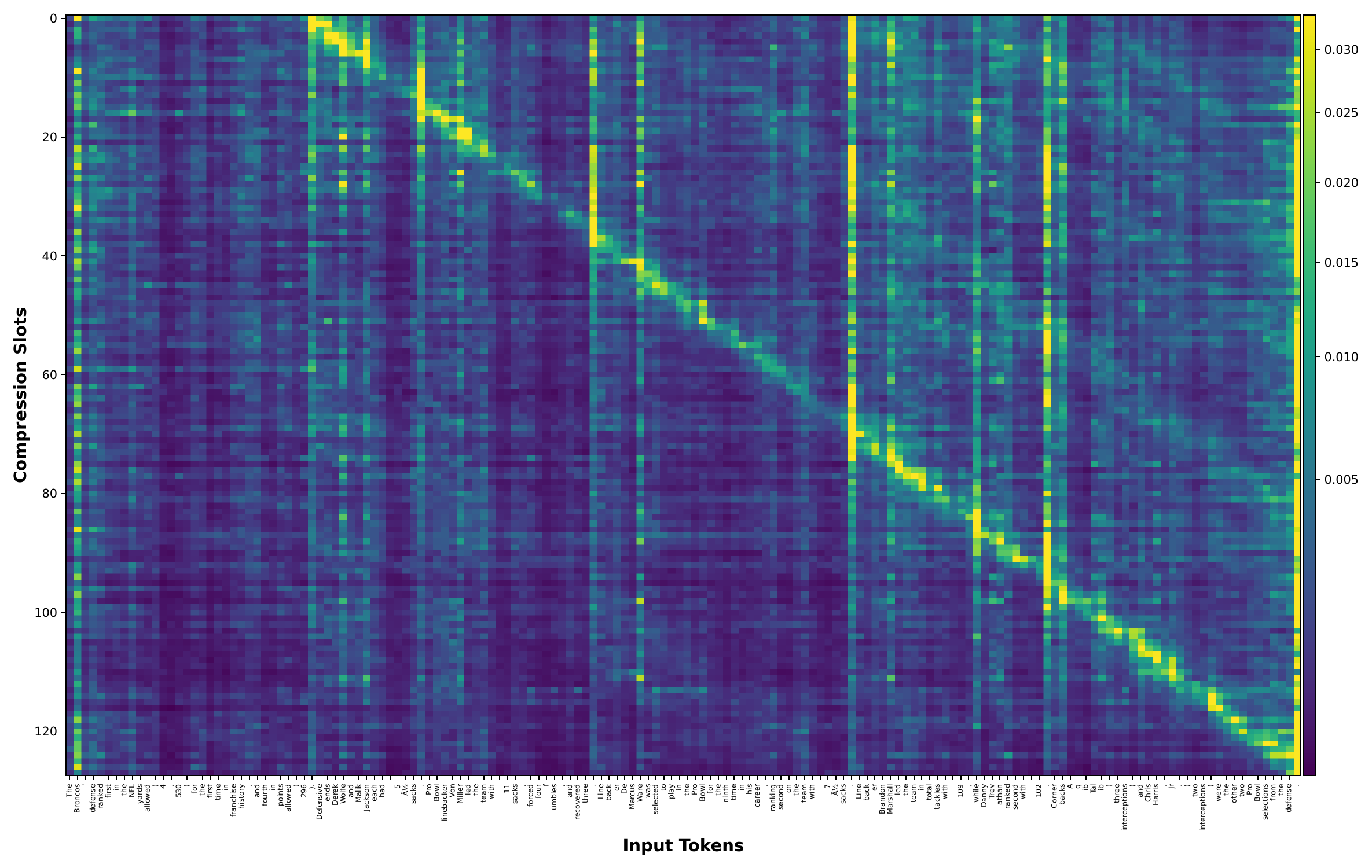}%
        \label{fig:width-icae}%
    }\hfill
    \subfloat[\method{}]{%
        \includegraphics[width=0.49\linewidth,height=0.28\textheight,keepaspectratio]{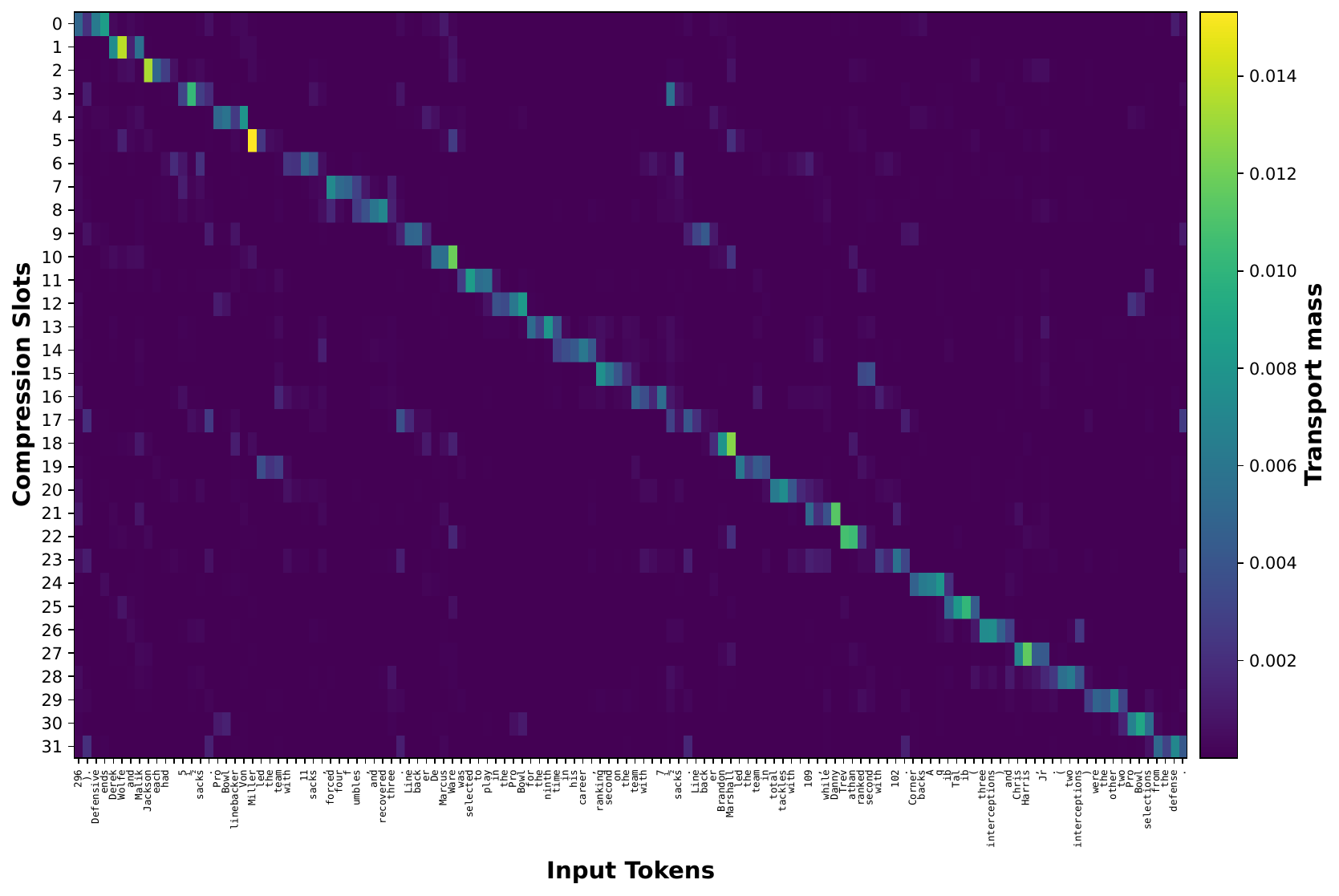}%
        \label{fig:width-ours}%
    }
    \caption{Last-layer attention heatmap of gist tokens produced by ICAE (left) and width-wise information transmission plan in \method{} (right).}
    \label{fig:width-plan}
\end{figure}

Figure~\ref{fig:width-ours} visualizes the learned width-wise transmission plan in \method{}. Compared with existing methods (e.g., ICAE in Figure~\ref{fig:width-icae}), \method{} exhibits a more structured allocation pattern: each compression slot is softly anchored to a contiguous local region of input tokens, while still retaining non-negligible connections to distant tokens. This enables \method{} to preserve the semantic order of the context without sacrificing the ability to capture long-range dependencies. In contrast, although ICAE tries to maintain an order bias (there is a diagonal pattern of high attention values in \cref{fig:width-icae}), its compression-token attention patterns are weakly localized. Moreover, it can be observed that some important tokens (e.g., entity ``Marcus'', number ``4.'') are ignored collectively by most of the compression tokens, but some entities (e.g., Broncos) are overly focused. We attribute these observations to the lack of coordinated allocation.

\subsection{Distribution Mismatch Induced by Layerwise Dilution}
\label{appendix:dist-mismatch}
In this subsection, we formalize the potential mismatch introduced by layerwise dilution. The key idea is that if the compression-token representations drift slightly at each layer, then these small discrepancies can accumulate with depth and enlarge the mismatch between the final compressor output and the representation space expected by the decoder.

Let $p(Z^{(\ell)})$ denote the empirical distribution of compression-token representations $\{\bm{z^{(\ell)}_k}\}_{k=1}^K$ at layer $\ell$, and let $p_{\text{dec}}$ denote the representation distribution expected by the decoder. We measure this mismatch using the Wasserstein distance $\mathcal{W}$.

\begin{proposition}[Accumulation of layerwise drift]
Assume $p(Z^{(0)}),\dots,p(Z^{(L)})$ and $p_{\text{dec}}$ all have finite first moments. Define the final compressor--decoder mismatch as
\begin{align}
\mathcal{M}
&:= \mathcal{W}\!\left(p(Z^{(L)}),\,p_{\text{dec}}\right).
\end{align}
Then
\begin{align}
\mathcal{M}
&\le \sum_{\ell=0}^{L-1} \mathcal{W}\!\left(p(Z^{(\ell+1)}),\,p(Z^{(\ell)})\right)
 + \mathcal{W}\!\left(p(Z^{(0)}),\,p_{\text{dec}}\right).
\label{eq:wd_mismatch}
\end{align}
\end{proposition}

\begin{proof}
By the triangle inequality for the Wasserstein distance~\citep{clement2008elementary},
\begin{align}
\mathcal{W}\!\left(p(Z^{(L)}),\,p_{\text{dec}}\right)
&\le \mathcal{W}\!\left(p(Z^{(L)}),\,p(Z^{(L-1)})\right)
+ \mathcal{W}\!\left(p(Z^{(L-1)}),\,p_{\text{dec}}\right).
\end{align}
Applying the same inequality recursively to the second term yields
\begin{align}
\mathcal{W}\!\left(p(Z^{(L-1)}),\,p_{\text{dec}}\right)
&\le \mathcal{W}\!\left(p(Z^{(L-1)}),\,p(Z^{(L-2)})\right)
+ \mathcal{W}\!\left(p(Z^{(L-2)}),\,p_{\text{dec}}\right),
\end{align}
and continuing this expansion down to layer $0$ gives
\begin{align}
\mathcal{M}
&\le \sum_{\ell=0}^{L-1} \mathcal{W}\!\left(p(Z^{(\ell+1)}),\,p(Z^{(\ell)})\right)
+ \mathcal{W}\!\left(p(Z^{(0)}),\,p_{\text{dec}}\right).
\end{align}
This proves the claim.
\end{proof}

The proposition shows that the final mismatch is upper-bounded by two terms: the initial mismatch at the compressor input, and the cumulative layerwise drift accrued across the encoder. Therefore, even if each layer introduces only a modest shift, the total drift can still grow with depth, increase $\mathcal{M}$, and make the compression objective harder to optimize. Together with the extensive experiments in Sections~3 and Appendix~A, this provides both theoretical and empirical support for the role of information dilution.



\subsection{Training}
\subsubsection{Optimization Behavior}
\begin{figure}[!t]
    \centering
    \includegraphics[width=0.72\columnwidth]{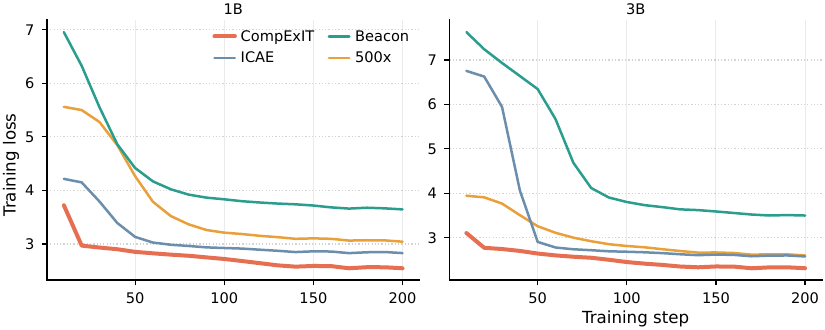}
    \caption{The training curves of baseline methods and \method{} under the next-token prediction task.}
    \label{fig:training_curve}
\end{figure}
We plot the NTP training curves in \cref{fig:training_curve}. \method{} starts from the lowest loss and converges faster to the best plateau, indicating a smaller compressor-decoder mismatch and easier optimization. This supports our core claim that layer aggregation and coordinated allocation produce compressed states that preserve useful information while remaining decoder-friendly.

\subsubsection{Training Details}
\label{appendix:training_details}
\cref{tab:ntp_ours,tab:ntp_beacon,tab:ntp_icae_500x} present the hyperparameters we use when training \method{} and the baseline methods.

\begin{table}[htbp]
\centering
\caption{NTP Training configuration for \method{}.}
\label{tab:ntp_ours}
\begin{tabularx}{0.75\linewidth}{@{}l>{\centering\arraybackslash}X@{}}
\toprule
\textbf{Item} & \textbf{\method{}} \\
\midrule
\multicolumn{2}{@{}l@{}}{\textbf{Run \& setup}} \\
Dataset & SlimPajama-6B \\
Number of Tokens & 1 Billion \\
Device & 4$\times$ NVIDIA-A100 \\
Precision & Bfloat16 \\
\midrule
\multicolumn{2}{@{}l@{}}{\textbf{Sequence \& compression}} \\
Context length & 512 \\
Generation length & 128 \\
Compression ratio & 4 \\
\midrule
\multicolumn{2}{@{}l@{}}{\textbf{Optimal-Transport (OT)}} \\
OT window size & 128 \\
OT iterations & 30 \\
OT projection dimension & 256 \\
Layerwise gate hidden size & 256 \\
\midrule
\multicolumn{2}{@{}l@{}}{\textbf{Optimization}} \\
Learning rate & $1\times 10^{-4}$ \\
Warmup ratio & 0.05 \\
Max grad norm & 20.0 \\
Batch size (per device) & 16 \\
Gradient accumulation steps & 32 \\
Actual Batch Size & 2048\\
Epochs & 1 \\
\bottomrule
\end{tabularx}
\end{table}

\begin{table}[htbp]
\centering
\caption{NTP Training configuration for ICAE and 500x.}
\label{tab:ntp_icae_500x}
\begin{tabularx}{0.75\linewidth}{@{}l>{\centering\arraybackslash}X>{\centering\arraybackslash}X@{}}
\toprule
\textbf{Item} & \textbf{ICAE} & \textbf{500x} \\
\midrule
\multicolumn{3}{@{}l@{}}{\textbf{Run \& setup}} \\
Dataset & SlimPajama-6B & SlimPajama-6B \\
Number of Tokens & 1 Billion & 1 Billion \\
Device & 4$\times$ NVIDIA-A100 & 4$\times$ NVIDIA-A100 \\
Precision & Bfloat16 & Bfloat16 \\
\midrule
\multicolumn{3}{@{}l@{}}{\textbf{Sequence \& compression}} \\
Context length & 512 & 512 \\
Generation length & 128 & 128 \\
Compression ratio & 4 & 4 \\
\midrule
\multicolumn{3}{@{}l@{}}{\textbf{Memory}} \\
Number of memory tokens & 128 & 128 \\
\midrule
\multicolumn{3}{@{}l@{}}{\textbf{Optimization}} \\
Learning rate & $3\times 10^{-5}$ & $3\times 10^{-5}$ \\
Warmup ratio & 0.05 & 0.05 \\
Max grad norm & 20.0 & 20.0 \\
    Batch size (per device) & 16 & 16 \\
Gradient accumulation steps & 16 & 32 \\
Actual Batch Size & 2048 & 2048 \\
Epochs & 1 & 1 \\
\midrule
\multicolumn{3}{@{}l@{}}{\textbf{LoRA (compressor)}} \\
Enabled & True & True \\
Rank $r$ & 128 & 128 \\
Alpha $\alpha$ & 32 & 32 \\
Dropout & 0.05 & 0.05 \\
Bias & none & none \\
Task type & CAUSAL\_LM & CAUSAL\_LM \\
\bottomrule
\end{tabularx}
\end{table}

\begin{table}[htbp]
\centering
\caption{NTP Training configuration for Activation Beacon.}
\label{tab:ntp_beacon}
\begin{tabularx}{0.75\linewidth}{@{}l>{\centering\arraybackslash}X@{}}
\toprule
\textbf{Item} & \textbf{Activation Beacon} \\
\midrule
\multicolumn{2}{@{}l@{}}{\textbf{Run \& setup}} \\
Dataset & SlimPajama-6B \\
Number of Tokens & 1 Billion \\
Device & 4$\times$ NVIDIA-A100 \\
Precision & Bfloat16 \\
\midrule
\multicolumn{2}{@{}l@{}}{\textbf{Sequence}} \\
Context length & 512 \\
Generation length & 128 \\
\midrule
\multicolumn{2}{@{}l@{}}{\textbf{Activation Beacon configuration}} \\
Beacon enabled & True \\
Beacon window / stride & 64 / 64 \\
Beacon ratio & 4 \\
Beacon attention & full-coverage \\
Attend previous beacons & True \\
Trainable beacon params & q, k, v \\
Beacon position & interleave \\
Grouping by stride & strict \\
\midrule
\multicolumn{2}{@{}l@{}}{\textbf{Optimization}} \\
Learning rate & $1\times 10^{-4}$ \\
Max grad norm & 20.0 \\
Batch size (per device) & 16 \\
Gradient accumulation steps & 32 \\
Actual Batch Size & 2048 \\
Epochs & 1 \\
\bottomrule
\end{tabularx}
\end{table}

\subsection{Broader Impact}
\label{appendix:broader-impact}
This work studies learned context compression for long-context language models. Its primary potential benefit is improved efficiency: effective compression can reduce memory usage, latency, and computational cost while preserving salient information from the original context. In turn, this may make document-grounded models more accessible to researchers and practitioners with limited hardware resources, and may reduce the energy required to deploy such systems at scale. More efficient long-context modeling may also broaden the practical use of language models in settings where rapid processing of long documents is important, such as literature review, writing assistance, and educational tools. Gist tokens can also function as encrypted memory or a communication medium for LLM agents while remaining unreadable to humans.

At the same time, context compression introduces important risks. Any compression method may discard rare but critical details, qualifiers, or provenance signals, which can lead to overconfident but incomplete answers. Such failures may be especially concerning in high-stakes applications such as medicine, law, science, or public policy, where omitted evidence can materially affect the correct conclusion. In addition, compressed-context models can still inherit social biases, factual errors, and privacy risks from their pretrained backbone and training data, and efficient processing of long documents could be misused for large-scale analysis of sensitive text. We therefore recommend task-specific evaluation, preserving access to the original context whenever possible, and avoiding deployment as a substitute for expert judgment in safety-critical settings.

\clearpage
\section*{NeurIPS Paper Checklist}

\begin{enumerate}

\item {\bf Claims}
    \item[] Question: Do the main claims made in the abstract and introduction accurately reflect the paper's contributions and scope?
    \item[] Answer: \answerYes{} 
    \item[] Justification: The abstract and introduction clearly state the paper's main claims: the two structural bottlenecks, the proposed explicit-transmission framework, and the empirical gains. These claims are supported by the method and experiments in \cref{sec:method} and \cref{section:exp_details}.

\item {\bf Limitations}
    \item[] Question: Does the paper discuss the limitations of the work performed by the authors?
    \item[] Answer: \answerYes{} 
    \item[] Justification: The paper explicitly discusses limitations in \cref{sec:conclusion}: we do not test much larger backbones, we evaluate task performance only up to 8k-token contexts and measure efficiency alone at 16k and 32k, \method{} has a higher peak memory than ICAE unless the context is compressed window by window, and our experiments focus on text-only models.

\item {\bf Theory assumptions and proofs}
    \item[] Question: For each theoretical result, does the paper provide the full set of assumptions and a complete (and correct) proof?
    \item[] Answer: \answerYes{} 
    \item[] Justification: The appendix \cref{appendix:dist-mismatch} includes a formal theoretical result (the proposition on accumulation of layerwise drift), states its assumption explicitly (finite first moments for the relevant distributions), and provides a complete proof based on the triangle inequality for Wasserstein distance.

    \item {\bf Experimental result reproducibility}
    \item[] Question: Does the paper fully disclose all the information needed to reproduce the main experimental results of the paper to the extent that it affects the main claims and/or conclusions of the paper (regardless of whether the code and data are provided or not)?
    \item[] Answer: \answerYes{} 
    \item[] Justification: We provide detailed description of our method in \cref{sec:method}. We provide details of our experiment setups in \cref{section:exp_details} and include the specifications to reproduce our training process in \cref{appendix:training_details}. 

\item {\bf Open access to data and code}
    \item[] Question: Does the paper provide open access to the data and code, with sufficient instructions to faithfully reproduce the main experimental results, as described in supplemental material?
    \item[] Answer: \answerYes{} 
    \item[] Justification: The evaluation datasets and base models are public, and the paper provides detailed experimental instructions. Our code is publicly available at \url{https://github.com/Jiangnan0522/ComprExIT}, with instructions to reproduce the main experimental results.

\item {\bf Experimental setting/details}
    \item[] Question: Does the paper specify all the training and test details (e.g., data splits, hyperparameters, how they were chosen, type of optimizer) necessary to understand the results?
    \item[] Answer: \answerYes{} 
    \item[] Justification: \cref{section:exp_details} specifies the base models, training/evaluation protocol, datasets, metrics, context length, and compression ratio. The appendix \cref{appendix:training_details} provides dataset statistics, implementation details, and per-method hyperparameter tables for training. Together, these details are sufficient to understand and interpret the reported results.

\item {\bf Experiment statistical significance}
    \item[] Question: Does the paper report error bars suitably and correctly defined or other appropriate information about the statistical significance of the experiments?
    \item[] Answer: \answerNo{} 
    \item[] Justification: The reported QA metrics are aggregate EM/F1 scores computed over evaluation splits containing many validation examples, and the latency analysis in the appendix reports the mean over 200 timed iterations. However, the main experimental results are still presented without run-to-run error bars, confidence intervals, or formal significance tests, so we answer this item conservatively.

\item {\bf Experiments compute resources}
    \item[] Question: For each experiment, does the paper provide sufficient information on the computer resources (type of compute workers, memory, time of execution) needed to reproduce the experiments?
    \item[] Answer: \answerYes{} 
    \item[] Justification: We provide the key compute details needed to reproduce the reported experiments in the paper: training is described with device information (e.g., 4$\times$NVIDIA A100), precision and optimization settings in \cref{section:exp_details,appendix:training_details}. The latency study specifies the hardware, memory, workload configuration, warm-up runs, and timed iterations (NVIDIA L40S 48GB GPU; 50 warm-up and 200 timed iterations) in \cref{appendix:exp-latency}. Together, these disclosures provide sufficient information about the compute workers and execution setup for reproduction.
    
\item {\bf Code of ethics}
    \item[] Question: Does the research conducted in the paper conform, in every respect, with the NeurIPS Code of Ethics \url{https://neurips.cc/public/EthicsGuidelines}?
    \item[] Answer: \answerYes{} 
    \item[] Justification: The work complies with the NeurIPS Code of Ethics, with particular attention to data integrity, transparency, and reproducibility.

\item {\bf Broader impacts}
    \item[] Question: Does the paper discuss both potential positive societal impacts and negative societal impacts of the work performed?
    \item[] Answer: \answerYes{} 
    \item[] Justification: \cref{appendix:broader-impact} discusses both positive impacts, such as lower memory, latency, and energy costs, and negative impacts and risks, such as omission of critical details, bias, privacy, and misuse.
    
\item {\bf Safeguards}
    \item[] Question: Does the paper describe safeguards that have been put in place for responsible release of data or models that have a high risk for misuse (e.g., pre-trained language models, image generators, or scraped datasets)?
    \item[] Answer: \answerNA{} 
    \item[] Justification: The submission does not release a high-risk generative model or scraped dataset, so controlled-release safeguards are not a central issue for this work.

\item {\bf Licenses for existing assets}
    \item[] Question: Are the creators or original owners of assets (e.g., code, data, models), used in the paper, properly credited and are the license and terms of use explicitly mentioned and properly respected?
    \item[] Answer: \answerYes{} 
    \item[] Justification: All assets used in this paper, including code, datasets, and pre-trained models, are appropriately attributed to their original creators or owners.

\item {\bf New assets}
    \item[] Question: Are new assets introduced in the paper well documented and is the documentation provided alongside the assets?
    \item[] Answer: \answerYes{} 
    \item[] Justification: We release our code and trained checkpoints at \url{https://github.com/Jiangnan0522/ComprExIT} under the Apache 2.0 license. The README documents the setup, training, evaluation, and reproduction of the baselines.

\item {\bf Crowdsourcing and research with human subjects}
    \item[] Question: For crowdsourcing experiments and research with human subjects, does the paper include the full text of instructions given to participants and screenshots, if applicable, as well as details about compensation (if any)? 
    \item[] Answer: \answerNA{} 
    \item[] Justification: The work does not involve crowdsourcing or research with human participants.

\item {\bf Institutional review board (IRB) approvals or equivalent for research with human subjects}
    \item[] Question: Does the paper describe potential risks incurred by study participants, whether such risks were disclosed to the subjects, and whether Institutional Review Board (IRB) approvals (or an equivalent approval/review based on the requirements of your country or institution) were obtained?
    \item[] Answer: \answerNA{} 
    \item[] Justification: The work does not involve crowdsourcing or research with human participants, so IRB-style review is not applicable.

\item {\bf Declaration of LLM usage}
    \item[] Question: Does the paper describe the usage of LLMs if it is an important, original, or non-standard component of the core methods in this research? Note that if the LLM is used only for writing, editing, or formatting purposes and does \emph{not} impact the core methodology, scientific rigor, or originality of the research, declaration is not required.
    \item[] Answer: \answerNA{} 
    \item[] Justification: The core method development in this research does not involve LLMs as any important, original, or non-standard components.

\end{enumerate}

\end{document}